\documentclass{article}
\usepackage{graphicx}
\usepackage{amsmath}
\usepackage{amssymb}
\usepackage{mathtools}
\usepackage{enumitem}
\usepackage{derivative}
\usepackage{amsthm}
\usepackage{thm-restate}
\usepackage[most]{tcolorbox}

\usepackage{amsthm}
\usepackage{algorithm}
\usepackage{algpseudocode}
\definecolor{MyDarkGreen}{RGB}{0, 100, 0}
\definecolor{mypurpose}{RGB}{151, 9, 252}
\definecolor{generativeblue}{RGB}{100, 143, 255}
\definecolor{method_figure_blue}{RGB}{37, 150, 190}
\PassOptionsToPackage{numbers, compress}{natbib}
\usepackage{float}
\usepackage{wrapfig}
\usepackage{bbm}
\usepackage{nicefrac}
\usepackage{algpseudocode}
\usepackage{titletoc}
\usepackage{algorithm}
\usepackage{extarrows}
\algrenewcommand\algorithmicrequire{\textbf{Input:}}
\algrenewcommand\algorithmicforall{\textbf{for all}}
\usepackage[most]{tcolorbox}

\tcbset{
  fluxbox/.style={
    colback=white,
    colframe=black!55,
    boxrule=0.5pt,
    arc=1pt,
    left=6pt,
    right=6pt,
    top=6pt,
    bottom=6pt,
  }
}

\newtcolorbox{graypropbox}{
  colback=gray!10,
  colframe=gray!60,
  boxrule=0.4pt,
  arc=1pt,
  left=3pt,
  right=3pt,
  top=2pt,
  bottom=2pt,
  boxsep=2pt
}
\newtcolorbox{blueemphbox}{
  colback=blue!10,
  colframe=blue!60,
  boxrule=0.4pt,
  arc=1pt,
  left=3pt,
  right=3pt,
  top=2pt,
  bottom=2pt,
  boxsep=2pt
}

\usepackage{caption}
 \usepackage[preprint]{neurips_2026}

\theoremstyle{plain}

\theoremstyle{definition}

\theoremstyle{remark}

\usepackage[utf8]{inputenc} 
\usepackage[T1]{fontenc}    
\usepackage[hidelinks]{hyperref}
\usepackage{url}            
\usepackage{booktabs}       
\usepackage{amsfonts}       
\usepackage{nicefrac}       
\usepackage{microtype}      
\usepackage{xcolor}         
\usepackage[capitalize,noabbrev]{cleveref}

\newcommand{\pdata}{p_{\mathrm{data}}}

\title{Generative Modeling with Flux Matching}

\author{%
  Peter Pao-Huang$^1$ \quad Xiaojie Qiu$^{1,2}$ \quad Stefano Ermon$^1$ \\
  $^1$Department of Computer Science, Stanford University\\
  $^2$Department of Genetics, Stanford University\\
  \texttt{\{peterph,xiaojie,ermon\}@stanford.edu}
}

\begin{document}

\maketitle

\begin{abstract}
  We introduce \textit{Flux Matching}, a new paradigm for generative modeling that generalizes existing score-based models to a broader family of vector fields that need not be conservative. Rather than requiring the model to equal the data score, the Flux Matching objective imposes a weaker condition that admits infinitely many vector fields whose stationary distribution is the data. This flexibility enables a class of generative models that cannot be learned under score matching, in which inductive biases, structural priors, and properties of the dynamics can be directly imposed or optimized. We show that Flux Matching performs strongly on high-dimensional image datasets and, more importantly, that our added freedom unlocks a range of applications including faster sampling, interpretable and mechanistic models, and dynamics that encode directed dependencies between variables. More broadly, Flux Matching opens a new dimension in generative modeling by turning the vector field itself into a design choice rather than a fixed target. Code is available at \url
  {https://github.com/peterpaohuang/flux_matching}.
\end{abstract}

\section{Introduction}
Many different vector fields produce diffusion processes with the same stationary distribution. Modern generative modeling \cite{song2019generative,song2020score,ho2020denoising}, however, canonically targets one particular vector field called the (Stein) score function, typically fit via score matching \cite{hyvarinen2005estimation,hyvarinen2005estimation,song2020score}, whose population loss is the Fisher divergence. Once learned, the score model can be used in Langevin dynamics or other gradient-based Markov chain Monte Carlo (MCMC) methods to generate samples from the target distribution $p_{\mathrm{data}}$. This score-based paradigm dominants current state-of-the-art image generation models \cite{rombach2022high,dhariwal2021diffusion,peebles2023scalable}, protein generation and design models \cite{watson2023novo,abramson2024accurate}, robotics \cite{chi2025diffusion,pearce2023imitating}, and others \cite{ho2022video,kong2020diffwave,gong2022diffuseq,corso2022diffdock,jing2022torsional,xu2022geodiff}.

\begin{wrapfigure}{r}{0.4\textwidth}
  \centering
  \vskip -0.15in
  \includegraphics[width=\linewidth]{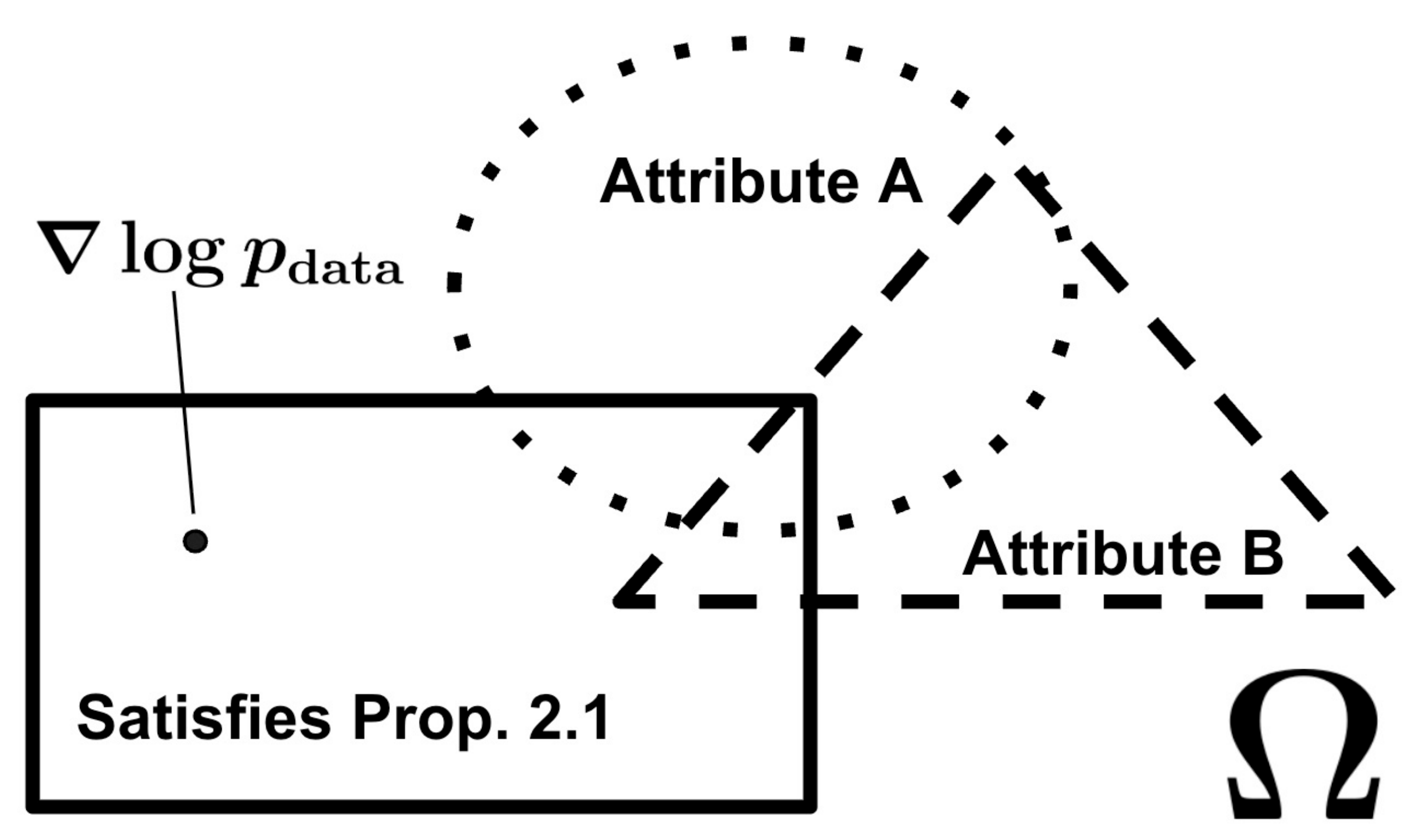}
  \caption{$\Omega$ is the space of vector fields $\in L^2(\pdata)$. Score matching learns $\nabla \log \pdata$, a single point in this space. In contrast, \textit{Flux Matching} can learn any vector field inside the rectangle.}
  \label{intro_figure}
  \vskip -0.1in
\end{wrapfigure}
The narrow focus on the score overlooks a large space of alternative vector fields whose diffusion processes share the same target distribution. We refer to these vector fields as \textit{generative vector fields}, or generative fields for short, with the Fokker–Planck equation (FPE) characterizing the full family \cite{horvat2024gauge}. \Cref{intro_figure} highlights the distinction: in the space of vector fields $\Omega$, score matching picks out a single point, $\nabla \log \pdata$, when any other point in the rectangle of generative fields characterized by the FPE is also equally valid. These non-score generative fields provide an extra degree of freedom for encoding useful attributes, illustrated abstractly by Attributes A and B in \Cref{intro_figure}. Concretely, they can capture directed dependencies between variables, impose mechanistic structure, improve smoothness or mixing, and produce dynamics that are meaningful in their own right rather than merely a means of sampling.

In this work, we propose \textit{Flux Matching}, a novel paradigm for learning generative vector fields beyond the score (aka any point inside the rectangle of \Cref{intro_figure}). Instead of requiring the model to equal \(\nabla\log \pdata\) pointwise, Flux Matching requires a weaker condition that only the divergence of the probability flux matches. This condition guarantees that the field generates the target distribution while leaving a nullspace of infinitely many valid generative fields. In order to compare the flux divergences in the same \(L^2(\pdata)\) geometry as the Fisher divergence---while preserving the non-score degrees
of freedom---we define a new statistical divergence called the
\textit{projected Fisher divergence} and derive a tractable Flux
Matching loss that computes it. 
\begin{figure}
  \centering
  \includegraphics[width=\textwidth]{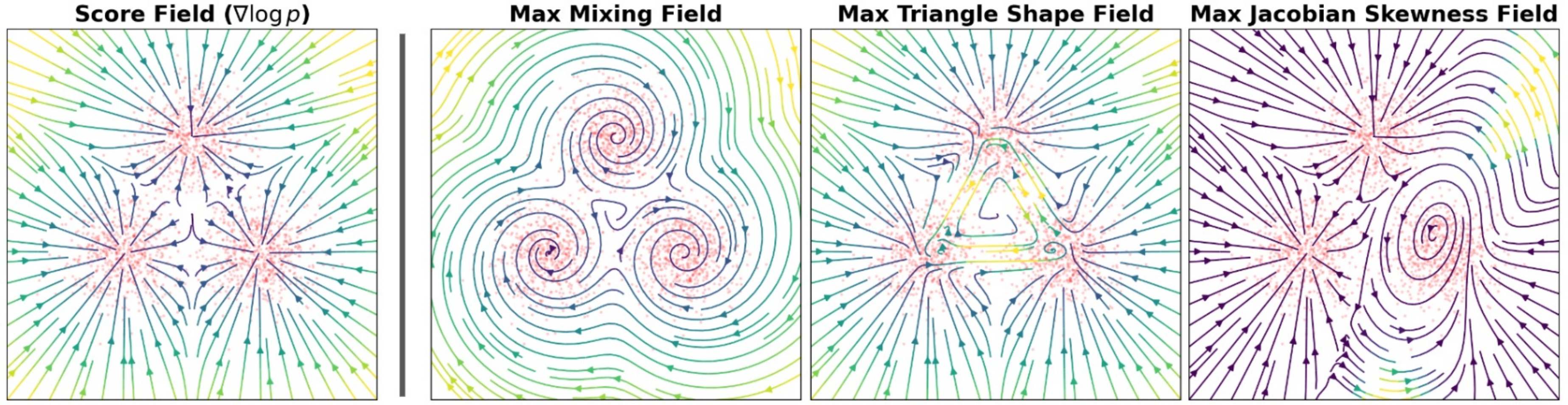}
  \caption{From \Cref{toy_controllable_experiments} where we maximize different vector field attributes that generate the same stationary distribution. (\textbf{Left}) Score function (\textbf{Right}) Alternative vector fields with useful properties.}
  \label{alternative_vector_fields_figure}
  \vskip -0.1in
\end{figure}
We further extend Flux Matching to the noise annealed setting used by diffusion models \cite{song2020score,song2019generative}. Rather than learning one field for the data distribution, we learn a continuum of fields for increasingly noise annealed distributions. Among the many valid vector fields that generate the target distribution, we also show how to select application-specific solutions either through architectural constraints or by adding regularizers that favor desired attributes.

Empirically, we show that Flux Matching is both scalable and useful in varying domains: (1) Flux Matching can be used as a standalone generative objective on high-dimensional image datasets such as CIFAR-10 and CelebA \(64\times 64\); (2) Flux Matching can learn faster mixing fields to accelerate sampling speed; (3) Flux Matching can fit interpretable RNA-velocity in single-cell genomics; and (4) Flux Matching can embed structural priors, such as directed temporal dependencies, directly into the generative field.

To summarize, our contributions are:
\begin{itemize}[leftmargin=2em, itemsep=0pt, topsep=0pt]
    \item We introduce \textit{Flux Matching}, a generative modeling paradigm that learns vector fields beyond the score by matching the divergence of the probability flux.
    \item We derive an efficient Flux Matching loss that preserves the Fisher divergence geometry.
    \item We extend Flux Matching to noise annealed generative modeling and show that it scales to complex image distributions.
    \item We demonstrate new use cases enabled by non-score generative fields, including faster mixing, interpretable fields like RNA velocity, and structured generative dynamics.
\end{itemize}
\section{Preliminaries}
Let $\pdata$ denote an unknown data distribution on $\mathbb{R}^d$, observed only through samples $\{x_i\}_{i=1}^n \sim \pdata$. A key goal of generative modeling is to learn a representation of $\pdata$ that allows us to generate new samples from this distribution. Existing approaches do this by either modeling the density itself \cite{dinh2016density,papamakarios2021normalizing}, an unnormalized density \cite{du2019implicit,lecun2006tutorial,gutmann2010noise,hinton2002training}, or the closely related score function \cite{hyvarinen2005estimation,song2020score,ho2020denoising}.
\subsection{The (Stein) score function $\nabla \log \pdata(x)$}
\textbf{Learning.} If the score was directly available, we could fit a vector field $f_{\theta} : \mathbb{R}^d \to \mathbb{R}^d$ by minimizing the Fisher divergence:
\begin{equation}
\label{fisher_divergence}
\mathcal J(\theta) = \mathbb{E}_{x \sim \pdata} \bigl[ \|f_{\theta}(x) - \nabla \log \pdata(x)\|^2 \bigr].
\end{equation}
However, the score $\nabla \log \pdata(x)$ is typically inaccessible because $\pdata$ itself is unknown. Score-based methods therefore rely on objectives that avoid direct access to this target, including implicit score matching \cite{hyvarinen2005estimation}, denoising score matching \cite{vincent2011connection}, and nonparametric kernel density estimation (KDE) approximations \cite{hyvarinen2005estimation,vincent2011connection}.

\textbf{Sampling.} Once a score estimator $f_{\theta} \approx \nabla \log \pdata$ has been learned, we can sample from $\pdata$ by simulating the diffusion $dx_t = f_{\theta}(x_t)\,dt + \sqrt{2}\,dW_t$, where $W_t$ is standard Brownian motion. In practice, we run gradient-based Markov chain Monte Carlo (MCMC) methods that discretize this diffusion. A standard example is unadjusted Langevin dynamics, $x_{k+1} = x_k + \eta f_\theta(x_k) + \sqrt{2\eta}\,\xi_k$ with $\xi_k \sim \mathcal{N}(0,I)$, which (under mild regularity) converges to $\pdata$ as $\eta \to 0$ and $k \to \infty$. For brevity, we say that vector field $f_{\theta}$ \textit{generates} $\pdata$ (or samples from $\pdata$) as shorthand for when the diffusion with drift $f_{\theta}$ has $\pdata$ as its stationary distribution.

\subsection{The Fokker--Planck equation}
To understand why the score is useful for sampling, it is helpful to look at how densities evolve under stochastic dynamics. Again, consider the diffusion $dx_t = f_{\theta}(x_t)\,dt + \sqrt{2}\,dW_t$. If $p_t$ denotes the time $t$ marginal density of $x_t$, then $p_t$ evolves according to the Fokker--Planck equation:
\begin{equation}
\label{fokker_planck_equation}
\frac{\partial p_t(x)}{\partial t}
=
-\Bigl[ \nabla \cdot \underbrace{\Bigl(p_t(x)f_{\theta}(x)\Bigr)}_{\text{flux of } f_{\theta}} - \nabla \cdot \underbrace{\Bigl(p_t(x)\nabla \log p_t(x)\Bigr)}_{\text{flux of score}}\Bigr].
\end{equation}
 Intuitively, \Cref{fokker_planck_equation} is a continuity equation in which the density evolves under the competing divergences of two \textit{probability fluxes}, the drift flux $p_t(x)f_{\theta}(x)$ carrying mass along the vector field $f_{\theta}$ and the score flux $p_t(x)\nabla \log p_t(x)$ encoding the smoothing effect of Brownian motion. At stationarity, the two forces perfectly balance and the density no longer changes in time, so $\partial_t p_t = 0$, which gives the following:
\begin{graypropbox}
\begin{restatable}{proposition}{propone}[Classical stationary Fokker--Planck characterization, Section 2.4 of \cite{pavliotis2014stochastic}]
\label{prop1}
$\pdata$ is a stationary distribution of the diffusion
$dx_t = f_{\theta}(x_t)\,dt + \sqrt{2}\,dW_t$ iff
\[
\begin{aligned}
\nabla \cdot \bigl(\pdata(x)f_{\theta}(x)\bigr)
- \nabla \cdot \bigl(\pdata(x) \nabla \log \pdata(x)\bigr) = 0 \qquad \text{for all } x.
\end{aligned}
\]
\end{restatable}
\end{graypropbox}
We therefore refer to vector fields $f_{\theta}$ satisfying \Cref{prop1} as \textit{generative vector fields}, since simulating the corresponding diffusion produces samples from $\pdata$. The usual choice is $f_{\theta}(x) = \nabla \log \pdata(x)$, for which the condition holds trivially. Importantly, however, \Cref{prop1} shows that the score is \emph{not} the only valid drift. Any vector field of the form
\begin{equation}
\label{flexibility_equation}
f_{\theta}(x) = \nabla \log \pdata(x) + v(x) \quad \text{with} \quad \nabla \cdot \bigl(\pdata(x)v(x)\bigr) = 0,
\end{equation}
has the same stationary distribution $\pdata$. In other words, there is generally a whole family of vector fields that preserve the same target distribution, and the score is only one particular member of this family. Pictorially, some alternative vector fields are shown in \Cref{alternative_vector_fields_figure}.

\section{Flux Matching}
\label{flux_matching_method_section}
We now shift from learning $f_{\theta}$ by matching the score to learning $f_{\theta}$ by matching the divergence of the \emph{probability flux} the score induces, which we call \emph{Flux Matching}. The motivation comes directly from the Fokker--Planck equation, where if our goal is to ensure that $f_\theta$ generates $\pdata$, then it is only necessary to match the flux divergence $\nabla \cdot (\pdata f_\theta)$ (from \Cref{prop1}) rather than to match the vector field pointwise. Consequently, Flux Matching allows learning the family of generative vector fields that need not be the score (e.g. \Cref{alternative_vector_fields_figure}).
\begin{blueemphbox}
\textbf{New Capability: \textit{Same} Distribution, \textit{Many} Dynamics}. Flux Matching enables vector fields to follow any dynamics that generate the target distribution.
\end{blueemphbox}
\subsection{Projected Fisher Divergence}
\label{projected_fisher_subsection}
Matching $\nabla\cdot(\pdata f_\theta)$ and $\nabla\cdot(\pdata \nabla\log \pdata)$ requires a geometry suited to learning vector fields since $f_{\theta}$ is the vector field we optimize. The most direct option, comparing the two flux divergences as scalar fields, is invariant to $\pdata$-preserving dynamics by construction, but moves one derivative beyond the $L^2(\pdata)$ vector field geometry of the Fisher divergence, making the objective sensitive to derivative-level artifacts \cite{chartrand2011numerical}. We therefore seek an objective that matches flux divergences \emph{within} the geometry of the Fisher divergence while remaining invariant to any perturbation $v$ with $\nabla\cdot(\pdata v)=0$. To this end, let $\Pi_{\mathrm{flux}} f$ denote the unique gradient field that satisfies $\nabla\cdot(\pdata\,\Pi_{\mathrm{flux}} f_{\theta})=\nabla\cdot(\pdata f_{\theta})$, and note that since $\nabla\log \pdata$ is already a gradient field, $\Pi_{\mathrm{flux}}(\nabla\log \pdata)=\nabla\log \pdata$. We define the \textit{projected Fisher divergence}:
\begin{equation}
\label{grad_proj_fisher_div}
\widetilde{\mathcal{J}}(\theta)
:=
\mathbb{E}_{x\sim \pdata}
\bigl[
\|\Pi_{\mathrm{flux}}f_\theta(x)-\nabla\log \pdata(x)\|^2
\bigr].
\end{equation}
This objective is invariant to any perturbation $v$ with $\nabla\cdot(\pdata v)=0$ because $\nabla\cdot(\pdata(f_\theta+v))=\nabla\cdot(\pdata f_\theta)$, and hence $\Pi_{\mathrm{flux}}(f_\theta+v)=\Pi_{\mathrm{flux}}f_\theta$ by the definition of $\Pi_{\mathrm{flux}}$. Directly computing \Cref{grad_proj_fisher_div}, however, is intractable in high dimensions.
\subsection{Flux Matching Loss}
\label{flux_loss_learning_objective}
We provide a scalable training objective with the same gradients as the
projected Fisher divergence of \Cref{grad_proj_fisher_div}. Let
\(u_\theta:=f_\theta-\nabla\log\pdata\) and
\(r_\theta:=\pdata^{-1}\nabla\cdot(\pdata u_\theta)
=\nabla\cdot u_\theta+u_\theta\cdot\nabla\log\pdata\). To bridge
\Cref{grad_proj_fisher_div} to our final loss, we show a series of key identities (proven in \Cref{proofs_appendix}):
\begin{equation}
\label{projected_fisher_to_ibp_chain}
\begin{aligned}
\widetilde{\mathcal J}(\theta)
&= \int_0^\infty
\mathbb E_{x_0\sim\pdata,\,x_t|x_0}
\!\left[
r_\theta(x_0)r_\theta(x_t)
\right]\,dt
&& \text{(\textbf{Step 1} via \Cref{lemma1})} \\
&= -\int_0^\infty
\mathbb E_{x_0\sim\pdata,\,x_t|x_0}
\!\left[
u_\theta(x_0)^\top
\frac{\partial x_t}{\partial x_0}^{\top}
\nabla_{x_t}r_\theta(x_t)
\right]\,dt
&& \text{(\textbf{Step 2} via \Cref{lemma2})} .
\end{aligned}
\end{equation}
\begin{wrapfigure}{h}{0.4\textwidth}
  \centering
  \vskip -0.3in
  \includegraphics[width=\linewidth]{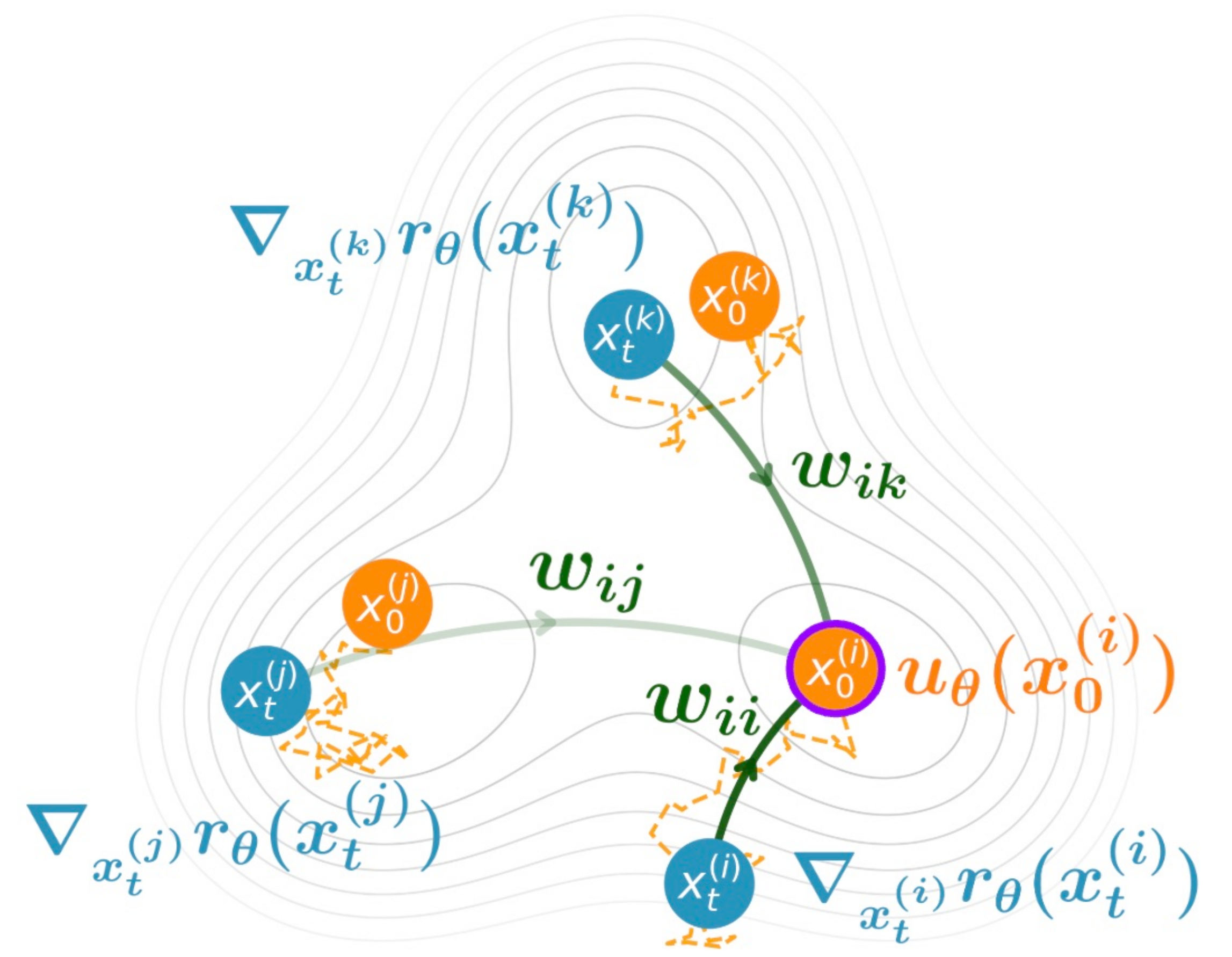}
  \caption{Geometric interpretation of Flux Matching. Colors are detailed in the accompanying \Cref{method_algorithm}.}
  \label{method_geometric}
  \vskip -0.15in
\end{wrapfigure}Intuitively, $r_\theta$ (aka the \textit{Langevin-Stein operator}) is invariant to any \(\pdata\)-preserving dynamics, but applies the local differential operator $\nabla \cdot$, which sits one derivative beyond the Fisher divergence geometry. \textbf{Step 1} closes this gap by using diffusion simulations $dx_t=\nabla\log\pdata(x_t)\,dt+\sqrt{2}\,dW_t$ to propagate the pointwise value of $r_\theta$ across nearby regions of the data distribution. Integrating the autocorrelation $r_\theta(x_0)r_\theta(x_t)$ over time $t$ accumulates these propagated values, effectively "undoing" the $\nabla \cdot$ operator and projecting $r_\theta$ back in the Fisher divergence geometry, as illustrated in \Cref{method_geometric}. \textbf{Step 2} uses integration by parts to convert the autocorrelation $r_\theta(x_0)r_\theta(x_t)$ into components $u_\theta(x_0)$ and $\partial x_t/\partial x_0^\top\,\nabla_{x_t} r_\theta(x_t)$. As a final \textbf{Step 3}, we apply a stop-gradient on $\partial x_t/\partial x_0^\top\,\nabla_{x_t} r_\theta(x_t)$. \Cref{lemma3} shows this preserves the gradient w.r.t.\ $\theta$ (up to a factor of $2$) while eliminating the need to backpropagate through the expensive $\partial x_t/\partial x_0^\top\,\nabla_{x_t} r_\theta(x_t)$ term. Sampling the simulation horizon \(t\sim q\) with $t \in [0,\infty)$ gives the resulting \textit{Flux Matching loss}:
\begin{equation}
\label{flux_matching_loss}
    \mathcal L_{\mathrm{flux}}(\theta)
:=
-\mathbb{E}_{\substack{t\sim q\\x_0\sim \pdata,\,x_t|x_0}}\!\left[
\frac{1}{q(t)}\,
u_\theta(x_0)^\top
\operatorname{sg}\!\left(
\frac{\partial x_t}{\partial x_0}^{\top}\nabla_{x_t}r_\theta(x_t)
\right)
\right] \quad \blacktriangleright\ \textbf{Flux Matching}
\end{equation}
where $\operatorname{sg}$ denotes stop-gradient and $x_t\mid x_0$ denotes a simulation chain from $x_0$ to $x_t$. We formalize the end-to-end connection between
the projected Fisher divergence (\Cref{grad_proj_fisher_div}) and the final
Flux Matching loss (\Cref{flux_matching_loss}) via the following:
\WFclear
\begin{graypropbox}
\begin{restatable}{theorem}{theoremone}
\label{theorem1}
Assume \(\pdata> 0\) on \(\mathbb R^d\) and boundary terms in integration-by-parts arguments vanish. Then,
\[
\nabla_\theta \widetilde{\mathcal{J}}(\theta)=2\,\nabla_\theta \mathcal L_{\mathrm{flux}}(\theta).
\]
\end{restatable}
\end{graypropbox}

\subsection{Estimating the Loss in Practice}
\label{loss_in_practice}
\begin{algorithm}[t]
  \caption{One Training Iteration of Flux Matching.}
  \label{method_algorithm}
  
  \vspace{0.2em}
  \footnotesize
  \textcolor{orange}{\(\boldsymbol{x_0}\)} denotes initial samples,
  \textcolor{mypurpose}{\(\boldsymbol{x_0^{(i)}}\)} highlights the selected initial sample,
  \textcolor{method_figure_blue}{\(\boldsymbol{x_t}\)} denotes simulated samples at time \(t\),
  and \textcolor{MyDarkGreen}{\(\boldsymbol{w_{im}}\)} denotes the weight of simulated sample \(m\) on \(x_0^{(i)}\). 
  \textit{(Note: Font colors match the visual nodes and edges in \Cref{method_geometric}).}
  \vspace{0.2em}

  \begin{algorithmic}[1]
    \footnotesize 
    \Require minibatch $\{\textcolor{orange}{\boldsymbol{x_0^{(i)}}}\}_{i=1}^B$, bandwidth $\sigma$, learnable vector field $f_\theta$

    \State Construct $\nabla \log \widehat p_\sigma$ via \Cref{batch_kde_score} and compute $\textcolor{orange}{\boldsymbol{
    u_\theta(x_0^{(i)})
    =
    f_\theta(x_0^{(i)})
    -
    \nabla \log \widehat p_\sigma(x_0^{(i)})
    }}
    \quad \text{for all } i$
    
    \State Sample shared simulation time $t \sim q$ on $[0,T]$ with $T=4\sigma^2$
    
    \State MCMC
    $\{\textcolor{orange}{\boldsymbol{x_0^{(i)}}}\}_{i=1}^B$
    to
    $\{\textcolor{method_figure_blue}{\boldsymbol{x_t^{(i)}}}\}_{i=1}^B$
    with $\nabla \log \widehat p_\sigma$ using \Cref{ou_mcmc}

    \State For selected initial sample $\textcolor{mypurpose}{\boldsymbol{x_0^{(i)}}}$, calculate \Cref{vr_endpoints}:
    $
    \widehat{
    \partial x_t/\partial x_0^{\top}\nabla r_\theta
    }(x_0^{(i)},t)
    :=
    \sum_{m=1}^B
    \textcolor{MyDarkGreen}{\boldsymbol{w_{im}}}
    \textcolor{method_figure_blue}{\boldsymbol{\nabla_{x_t^{(m)}} r_\theta(x_t^{(m)})}}
    $

    \State Apply Step 4 for every initial sample \(i=1,\dots,B\)

    \State Form $\mathcal{L}_{\mathrm{flux}}$ from \Cref{flux_matching_loss} and update $\theta$
  \end{algorithmic}
\end{algorithm}
\textbf{Approximating $\nabla \log \pdata$}. We replace the unknown score $\nabla \log \pdata$ with a nonparametric score approximation \cite{hyvarinen2005estimation} by building a KDE of the minibatch:
\begin{equation}
\label{batch_kde_score}
\nabla \log \widehat p_\sigma(x)
=
\left(\sum_{i=1}^B \exp\!\left(-\|x-x_i\|^2 \big/ 2\sigma^2\right)(x_i-x)\right)
\big/
\left(\sigma^2 \sum_{i=1}^B \exp\!\left(-\|x-x_i\|^2 \big/ 2\sigma^2\right)\right),
\end{equation}
which is asymptotically unbiased in the usual large-batch, vanishing-bandwidth regime and is a common technique used in modern generative models \cite{xu2023stable, song2023consistency, geng2025mean, lai2026unified}.

\textbf{Simulating $x_t$ from $x_0$}. We sample the simulation horizon $t\sim q$ where $q$ is supported on $[0,\infty)$. Simulating arbitrarily large horizons is computationally infeasible, however, so we truncate the support of $q$ to $[0,T]$. In practice, we find that defining $q$ to be either a truncated uniform or exponential and setting $T = 4\sigma^2$ is sufficient (justification is provided in \Cref{appendix_importance_sampler}). Given $t$, we run $4$ MCMC steps with step size $h=\frac{1}{4}t$ starting from $x_0$. To enable stable large-step sampling, we use an exponentially integrated Langevin update:
\begin{equation}
\label{ou_mcmc}
x_{k+1}
=
\mu_k
+
e^{-h} (x_k-\mu_k)
+
\sigma\sqrt{1-e^{-2h}} \xi_k,
\quad
\mu_k
=
x_k+\sigma^2 \nabla \log p_\sigma(x_k),
\quad
\xi_k \sim \mathcal N(0,I).
\end{equation}
Note that the computational cost of calculating \Cref{ou_mcmc} many times is \textit{negligible} compared to running $f_{\theta}$ once since $\nabla \log p_\sigma$ is a closed-form KDE approximation. 

\textbf{Estimating $\smash{\nicefrac{\partial x_t}{\partial x_0}^{\top}}\nabla_{x_t}r_\theta(x_t)$.}
One could backpropagate through each simulated chain to obtain the pathwise term $\smash{\nicefrac{\partial x_t}{\partial x_0}^{\top}}\nabla_{x_t}r_\theta(x_t) = \nabla_{x_0}r_\theta(x_t)$. However, this term appears inside a conditional expectation over $x_t\mid x_0$ (aka an expectation over simulation paths from $x_0$). Therefore, for each initial point $\smash{x_0^{(i)}}$ and time $t$, we can reduce variance by estimating the conditional expectation using all simulated chains $\smash{\{x_t^{(j)}\}_{j=1}^B}$ from the current minibatch (particularly helpful is specific regimes of $\sigma$, which we detail in \Cref{noise_annealed_details_appendix}), rather than only using a single chain $\smash{x_t^{(i)}}$ generated from $\smash{x_0^{(i)}}$. 

Unlike the same chain sensitivity \(\smash{\partial x_t^{(i)}/\partial x_0^{(i)}}\), the cross-chain sensitivity \(\smash{\partial x_t^{(j)}/\partial x_0^{(i)}}\) where $i \neq j$ is unavailable, so we approximate it with Gaussian transition weights \(w_{ij}\) normalized over minibatch endpoints \(j\) (further details in \Cref{appendix_variance_reduction_estimation}). Specifically, for each \(\smash{x_0^{(i)}}\), we replace \(\smash{\partial x_t/\partial x_0^{\top}}\nabla_{x_t}r_\theta(x_t)\) with the variance-reduced estimator:
\begin{equation}
\label{vr_endpoints}
\widehat{(\partial x_t/\partial x_0)^{\top}\nabla r_\theta}(x_0^{(i)},t) := \sum_{j=1}^B w_{ij}(t)\,\nabla_{x_t^{(j)}} r_\theta\!\left(x_t^{(j)}\right).
\end{equation}
Finally, the divergence term in $r_\theta = \nabla\cdot u_\theta + u_\theta\cdot \nabla\log p$ is approximated with a single-sample Hutchinson trace estimator \cite{hutchinson1989stochastic}.

\subsection{Extension to Noise Annealed Generative Fields}
\label{noise_conditioned_section}
Rather than learning a single vector field at the data distribution, diffusion
models \cite{song2020score, song2019generative,lipman2022flow, ho2020denoising, sohl2015deep} learn score fields over a continuum of noise annealed distributions
\(\{p_\sigma\}_{\sigma\sim\mathcal P}\) where \(\mathcal P\) denotes the sampling distribution for noise levels used during training. Here, \(p_\sigma=\pdata*\mathcal N(0,\sigma^2 I)\). Flux Matching extends to this setting
by applying the same objective independently at each noise level. Let
\(f_\theta^\sigma(x):=f_\theta(x,\sigma)\), and let
$q_{\sigma}$ denote the $\sigma$-dependent importance sampler over simulation horizons (can be learned via \Cref{appendix_importance_sampler}).
We write
\begin{equation}
\mathcal L_{\mathrm{flux}}^\sigma(\theta)
:=
\mathcal L_{\mathrm{flux}}
\bigl(
\theta;\,
p_\sigma,\,
q_{\sigma},\,
f_\theta^\sigma
\bigr),
\end{equation}
where the right-hand side denotes \Cref{flux_matching_loss} with
\(p\) replaced by \(p_\sigma\), \(f_\theta\) replaced by
\(f_\theta^\sigma\), \(q\) replaced by $q_{\sigma}$, and all
noise-dependent quantities evaluated at the same \(\sigma\) 
(e.g., the truncation horizon \(T=4\sigma^2\)).
The noise annealed Flux Matching objective is then
\begin{equation}
\label{noise_annealed_flux_matching}
    \mathcal L_{\mathrm{flux\text{-}noise}}(\theta,\eta)
:=
\mathbb E_{\sigma\sim\mathcal P}
\left[
\mathcal L_{\mathrm{flux}}^\sigma(\theta) / \exp(s_\eta(\sigma))
+
s_\eta(\sigma)
\right].
\end{equation}
\(s_\eta(\sigma)\) is a learned normalizer (single-layer MLP) that reweighs losses from
different noise levels to be on comparable scales \cite{karras2024analyzing} and is simultaneously trained with the main network.

\subsection{Sampling \& Likelihood Computation}
\label{sampling_and_likelihood}
Training and sampling are \textit{fully decoupled}. Although \(f_\theta\)
is learned through Flux Matching, at sampling time it can replace the
score term in standard score-based samplers (e.g. unadjusted Langevin dynamics) with no algorithmic changes and additional cost. Similarly, models learned via noise annealed Flux Matching can be used with reverse diffusion and probability-flow ODE samplers by simply replacing the noise conditioned score with \(f_\theta^\sigma\). For probability-flow ODE sampling, likelihoods can
also be computed with the usual instantaneous change-of-variables
formula \cite{chen2018neural}. See \Cref{score_replacement} for a formal statement.

\subsection{Learning Useful Generative Vector Fields}
Flux Matching learns a family of vector fields that generate the same target
distribution. The remaining task is to choose an element of this family with
properties useful for the application. We provide two general strategies:

\textbf{(1) Application Specific Loss.} Augment the Flux Matching
    objective with a loss $L_{\mathrm{app}}$ that encourages the desired properties, e.g.,
    \(\mathcal L_{\mathrm{flux}}+\sum_i\lambda_{\mathrm{app},i}
    \mathcal L_{\mathrm{app},i}\). For example, adding
    \(\lambda_{\mathrm{L2}}\|f_\theta\|_{L^2(p)}^2\) recovers the
    score function when minimized. In the noise annealed setting, nonnegative application losses can be normalized across noise levels using an equivalent learned normalizer as \Cref{noise_annealed_flux_matching}. For signed losses, we instead can normalize within discrete \(\sigma\)-buckets using running statistics, as in \cite{choi2022density}.
    
\textbf{(2) Model Parameterization.} Desired attributes can also be built
    directly into the architecture used to represent \(f_\theta\). For example,
    an attention mask in a transformer can enforce directed relationships among variables \cite{vaswani2017attention}.

\Cref{experiment_section} instantiates these two strategies in different settings, showing how Flux Matching can select useful generative fields while preserving the same
target distribution. 
\section{Applications of Flux Matching}
\label{experiment_section}
We evaluate Flux Matching across five settings that highlight unique applications of the method. The first two experiments use the Flux Matching loss from \Cref{flux_matching_loss}: \Cref{toy_controllable_experiments} isolates the main controllability benefit of Flux Matching on a toy distribution, and \Cref{interpretable_generative_fields} shows that Flux Matching can fit biologically interpretable vector fields. The remaining three experiments use the noise annealed objective from \Cref{noise_annealed_flux_matching}: \Cref{unrestricted_generative_fields} tests Flux Matching as a standalone image generation objective, \Cref{fast_mixing_generative_fields} leverages Flux Matching to optimize fields for faster sampling, and \Cref{structured_generative_fields} uses Flux Matching to impose directed structure between variables.
\subsection{\textit{Controllable} Generative Fields}
\label{toy_controllable_experiments}
\begin{figure}
  \centering
  \includegraphics[width=\textwidth]{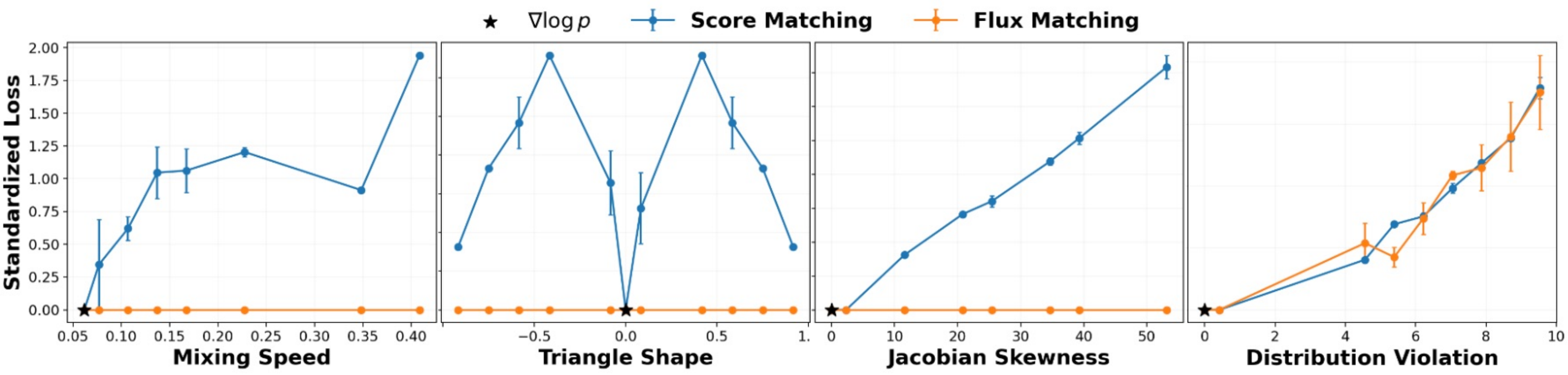}
  \caption{Normalized score matching and Flux Matching losses as we vary properties of the vector field on a Gaussian mixture. The black star denotes the attribute value of the score field $\nabla \log \pdata$. The first three panels vary distribution preserving fields with different values of a chosen attribute; the last panel varies fields that violate the target stationary distribution.}
  \label{app1_figure}
\end{figure}
One advantage of Flux Matching is that it exposes
distribution-preserving degrees of freedom for control. Many vector fields
share the same stationary distribution, and their differences govern
properties of the generative process such as mixing rate, circulation pattern,
and reversibility. Score matching targets $\nabla \log \pdata$ as the unique
correct field and penalizes any deviation, even ones that leave the
distribution unchanged. Flux Matching instead treats the entire
distribution preserving family as equivalent.

\textbf{Setup.} On a 2D three-component Gaussian mixture, we construct three
one-parameter families of distribution preserving perturbations of the score
field, indexed by attributes we call \textit{mixing speed},
\textit{triangle shape}, and \textit{Jacobian skewness}. For each perturbed
field we compute the score matching and Flux Matching losses, alongside a
distribution-violation metric that is zero exactly on the distribution preserving family. Since the two objectives have different raw
scales, we report standardized losses. Full definitions and construction
details appear in \Cref{app1_details_appendix}.

\textbf{Results.} \Cref{app1_figure} displays the outcomes. In the first
three panels, increasing the perturbation magnitude drives the score matching
loss up immediately while the Flux Matching loss stays at exactly zero. Practitioners can therefore tune these degrees of freedom to
shape the dynamics, for example to increase mixing, enforce triangular
circulation, or induce nonreversible structure, without changing the target
density. The fourth panel perturbs the field outside the
distribution preserving family, and the Flux Matching loss now rises with the
degree of distribution violation. Flux Matching is not flat everywhere. It is
flat precisely on the family of vector fields sharing the target stationary
distribution.

\subsection{\textit{Interpretable} Generative Fields}
\label{interpretable_generative_fields}
\begin{figure}
  \centering
  \setlength{\tabcolsep}{3pt}
  \renewcommand{\arraystretch}{1.05}
    \vskip -0.1in
  \begin{minipage}[t]{0.3\textwidth}
    \centering
    \vspace{0pt}
    \par\vspace{0.3em}
    \includegraphics[width=\linewidth]{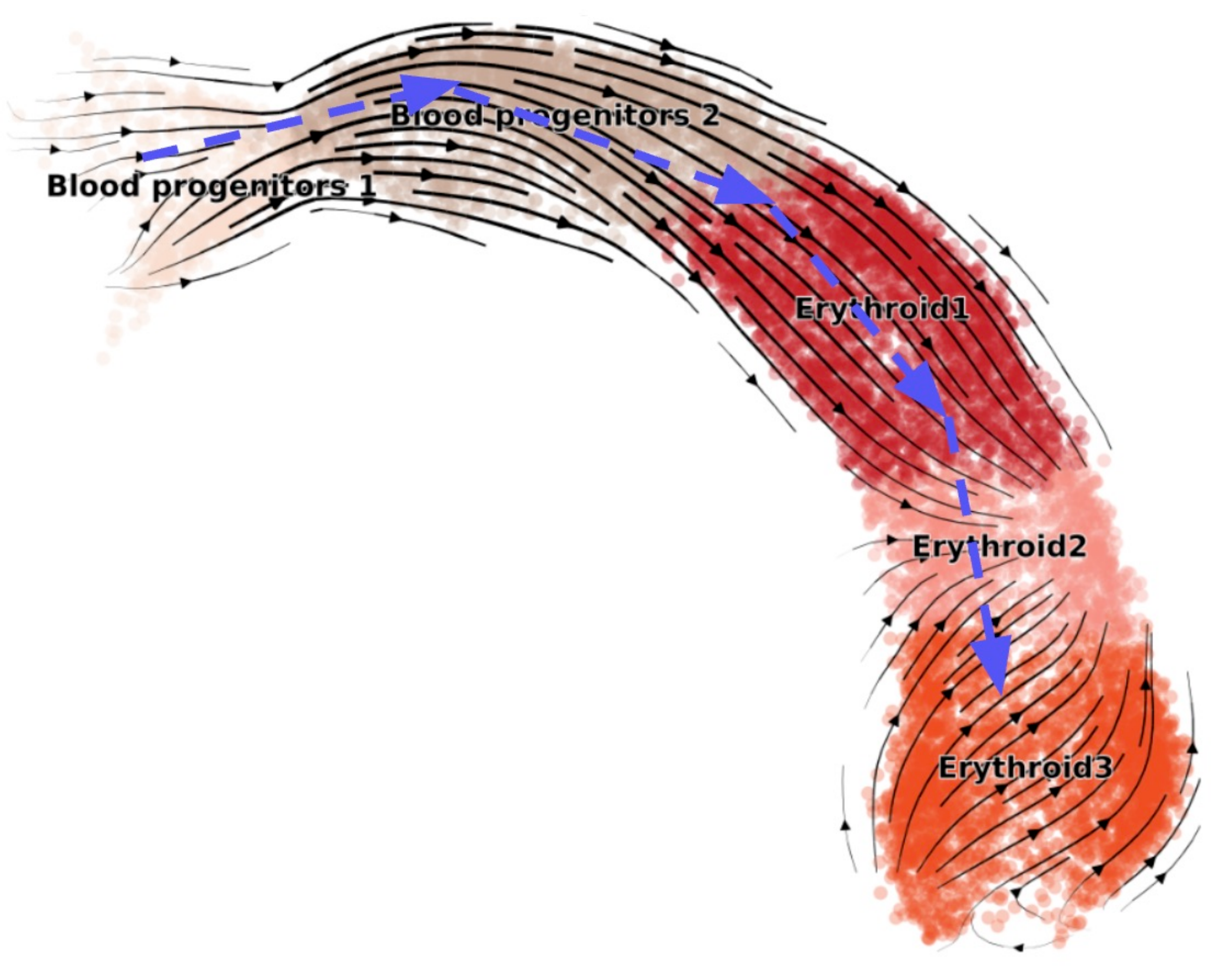}
  \end{minipage}
  \hfill
  \begin{minipage}[t]{0.67\textwidth}
    \centering
    \vspace{0pt}
    \par\vspace{0.3em}

    \footnotesize
    \begin{tabular}{@{}lcc@{}}
    \toprule
    \textbf{Dataset}
    & \textbf{Flux Matching}
    & \textbf{scVelo \cite{bergen2020generalizing}} \\
    & CBC $\uparrow$ / Consist. $\uparrow$
    & CBC $\uparrow$ / Consist. $\uparrow$ \\
    \midrule
    \textbf{Pancreas}     & 0.202 / \textbf{0.972} & \textbf{0.330} / 0.821 \\
    \textbf{Gastrulation} & \textbf{0.611} / \textbf{0.991} & -0.639 / 0.877 \\
    \textbf{Dentategyrus} & \textbf{0.284} / \textbf{0.981} & -0.084 / 0.791 \\
    \textbf{Bone Marrow}  & \textbf{0.177} / \textbf{0.939} & -0.789 / 0.857 \\
    \textbf{Hindbrain}    & \textbf{0.345} / \textbf{0.897} & 0.332 / 0.874 \\
    \bottomrule
    \end{tabular}
  \end{minipage}

  \caption{
  \textbf{(Left)} Learned RNA velocity using Flux Matching; blue arrows indicate ground-truth biological progression between cell-types.
  \textbf{(Right)} CBC and consistency means across datasets.
  }
  \label{fig:rna_velocity_results}
  \vskip -0.1in
\end{figure}
A second practical advantage of Flux Matching is that the vector field $f_\theta$ may
be of any parametric family, including ones whose parameters carry scientific
meaning. This enables \emph{interpretable} generative dynamics. Rather than
learning a black-box neural field, we constrain $f_\theta$ to a structured form
chosen by domain experts and fit its parameters directly from data. We
illustrate this on RNA velocity \cite{la2018rna}, a problem in single-cell biology where the
admissible vector fields are prescribed by a known mechanistic model.

\textbf{RNA velocity.} From a single static snapshot of a cell population, RNA
velocity aims to infer the direction in which each cell is moving through
gene expression space. For each of $G$ genes, two quantities are measured per
cell, namely an immature transcript $u_g$ and its mature form $s_g$. A standard
biophysical model \cite{bergen2020generalizing} prescribes the ordinary differential equation (ODE)
\begin{equation}
\label{rna_kinetic_field}
\frac{d}{d\tau}
\begin{pmatrix} u_g \\ s_g \end{pmatrix}
=
\begin{pmatrix}
\alpha_g(\tau) - \beta_g\, u_g \\
\beta_g\, u_g - \gamma_g\, s_g
\end{pmatrix},
\end{equation}
where the rates $\alpha_g, \beta_g, \gamma_g$ are biologically meaningful
(transcription, splicing, and degradation, respectively). A dominant method, scVelo \cite{bergen2020generalizing}, fits these rates per gene using
an EM-style latent-variable procedure that is known to be sensitive to
initialization.

\textbf{Flux Matching as a drop-in trainer.} We keep the biophysical model
\eqref{rna_kinetic_field} unchanged but replace the bespoke EM fit with
gradient descent on $\mathcal{L}_{\mathrm{flux}}$. Concretely, we concatenate
the per-gene fields across $G=2000$ genes into a full cell-state vector field
and optimize the scalar parameters jointly. The structured ODE restricts the
admissible vector fields, while Flux Matching supplies the training objective.

\textbf{Results.} \Cref{fig:rna_velocity_results} reports two standard RNA
velocity metrics. \emph{Cross-boundary correctness} (CBC) measures how well
predicted velocities align with known transitions between cell types, and
\emph{consistency} measures whether nearby cells receive similar velocity
directions. Across five real single cell datasets, Flux Matching improves consistency on
all five and CBC on four out of five, under the same parametric family as
scVelo. Because the model class is unchanged, the gains are attributable to
the fitting procedure rather than to added expressivity. We foresee that Flux Matching can be applied to other newly developed RNA velocity models with more sophisticated biological parameterizations. 

\subsection{\textit{Unrestricted} Generative Fields}
\label{unrestricted_generative_fields}

Flux Matching's main value is in imposing structure on the learned vector
field, but the Flux Matching loss is also viable as a standalone training objective on
complex high-dimensional distributions. We verify this in the
\textit{unrestricted} (vanilla) setting, where no additional field property is
optimized.

\textbf{Setup.} We evaluate on CIFAR10 ($3 \times 32 \times 32$) and CelebA
($3 \times 64 \times 64$). We train a standard UNet architecture from
\cite{futurexiang2023diffusion} with the noise annealed Flux Matching objective
in \Cref{noise_annealed_flux_matching}, using the EDM noise level distribution
\cite{karras2022elucidating}, for $500{,}000$ steps. As a baseline, we train
noise annealed denoising score matching (DSM) with the same architecture and
hyperparameters.

\textbf{Results.} The top half of \Cref{tab:unrestricted_generation} shows
that Flux Matching performs strongly on both datasets,
\begin{wraptable}{r}{0.5\textwidth}
    \centering
    \footnotesize
    \vskip -0.1in
    \caption{Unconditional generation performance and training-time efficiency.}
    \label{tab:unrestricted_generation}
    \renewcommand{\arraystretch}{0.85} 
    \setlength{\tabcolsep}{4pt}        
    \begin{tabular}{llccc}
        \toprule
        \textbf{Dataset} & \textbf{Model} & \multicolumn{3}{c}{\textbf{Performance}} \\
        \cmidrule(lr){3-5}
        & & \textbf{FID ($\downarrow$)} & \textbf{IS ($\uparrow$)} & \textbf{NLL (bpd, $\downarrow$)} \\
        \midrule
        CIFAR10 & DSM  & 4.74 & 8.52 & 3.16 \\
                & Flux & 9.06 & 8.54 & 3.26 \\
        CelebA  & DSM  & 2.41 & -    & 2.03 \\
                & Flux & 7.07 & -    & 2.17 \\
        \midrule
        \textbf{Dataset} & \textbf{Model} & \multicolumn{3}{c}{\textbf{Efficiency}} \\
        \cmidrule(lr){3-5}
        & & \multicolumn{2}{c}{\textbf{Speed (it/s)}} & \textbf{Memory/GPU (G)} \\
        \midrule
        CIFAR10 & DSM  & \multicolumn{2}{c}{11.63} & 2.79 \\
                & Flux & \multicolumn{2}{c}{4.01}  & 5.69 \\
        CelebA  & DSM  & \multicolumn{2}{c}{7.20}  & 7.79 \\
                & Flux & \multicolumn{2}{c}{1.77}  & 22.67 \\
        \bottomrule
    \end{tabular}
\end{wraptable}demonstrating that the loss alone scales to realistic
high-dimensional image distributions. The remaining FID gap to DSM is
unsurprising since DSM has benefited from many engineering iterations
specifically aimed at optimizing FID, whereas Flux Matching is evaluated here
as a first-pass implementation of a new learning objective. The bottom half
shows that Flux Matching is roughly \(3\)--\(4\times\) slower than DSM during
training and uses about \(2\)--\(3\times\) more memory. This overhead is
incurred only during training. At sampling time, the learned field can be
used in the same samplers as a score model. As noted earlier, the main
motivation for Flux Matching is not to replace DSM in this unrestricted
setting but to enable dynamics with useful properties, as shown in the next
two experiments.

\subsection{\textit{Fast Mixing} Generative Fields for Accelerated Sampling}
\label{fast_mixing_generative_fields}
\begin{figure}
  \centering
  \includegraphics[width=0.95\textwidth]{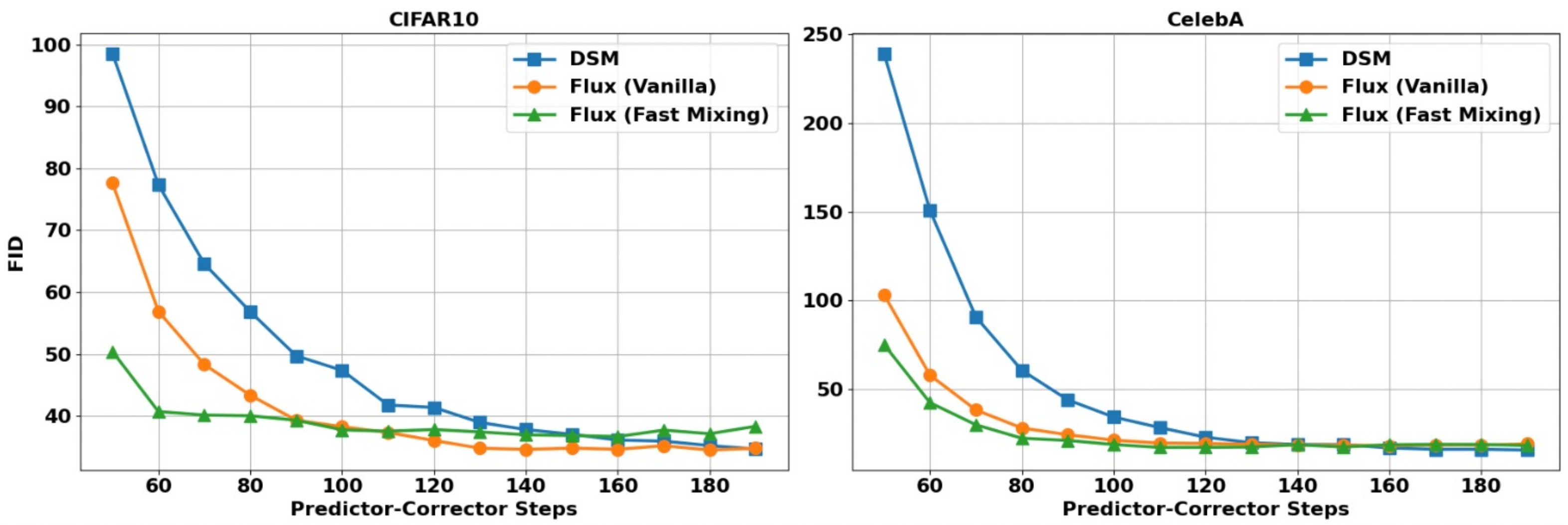}
  \caption{FID (calculated with $1$K generated samples) as a function of the number of sampling steps.}
  \label{fast_mixing_figure}
  \vskip -0.2in
\end{figure}
Fast mixing fields can accelerate sampling by requiring fewer sampling steps to
converge to the target distribution. Score-based Langevin dynamics is reversible
and known to mix slowly \cite{hwang2005accelerating, rey2015irreversible, duncan2016variance},
so score matching cannot exploit this property while Flux Matching can. Since
mixing time itself is intractable to optimize directly, we minimize a proxy
defined in \Cref{app4_details_appendix}.

\textbf{Setup.} We reuse the training and model setup of
\Cref{unrestricted_generative_fields} but add the mixing proxy with weight
$\lambda_{\mathrm{mixing}}=0.01$, giving
$\mathcal{L}_{\mathrm{flux-fast}}=\mathcal{L}_{\mathrm{flux\text{-}noise}} + \lambda_{\mathrm{mixing}}\,\mathcal{L}_{\mathrm{mixing}}$.
After training $f_\theta^\sigma$ with $\mathcal{L}_{\mathrm{flux-fast}}$ on CIFAR10 and CelebA, we evaluate FID on
$1$K generated samples across different numbers of sampling steps and compare
against vanilla Flux Matching and DSM (noise annealed versions).

\textbf{Results.} As shown in \Cref{fast_mixing_figure}, fast-mixing Flux
Matching reaches a reasonable FID with substantially fewer sampling steps than
vanilla Flux Matching and DSM on CIFAR10; on CelebA, the gain over vanilla Flux
Matching is modest. Interestingly, Flux Matching without fast mixing reaches a reasonable FID with fewer sampling steps than DSM on both datasets, suggesting that Flux Matching may already learn fields whose sampling dynamics are easier to mix.

\subsection{\textit{Embedding Structure} in Generative Fields}
\label{structured_generative_fields}

\begin{figure}
  \centering
  \includegraphics[width=1.0\textwidth]{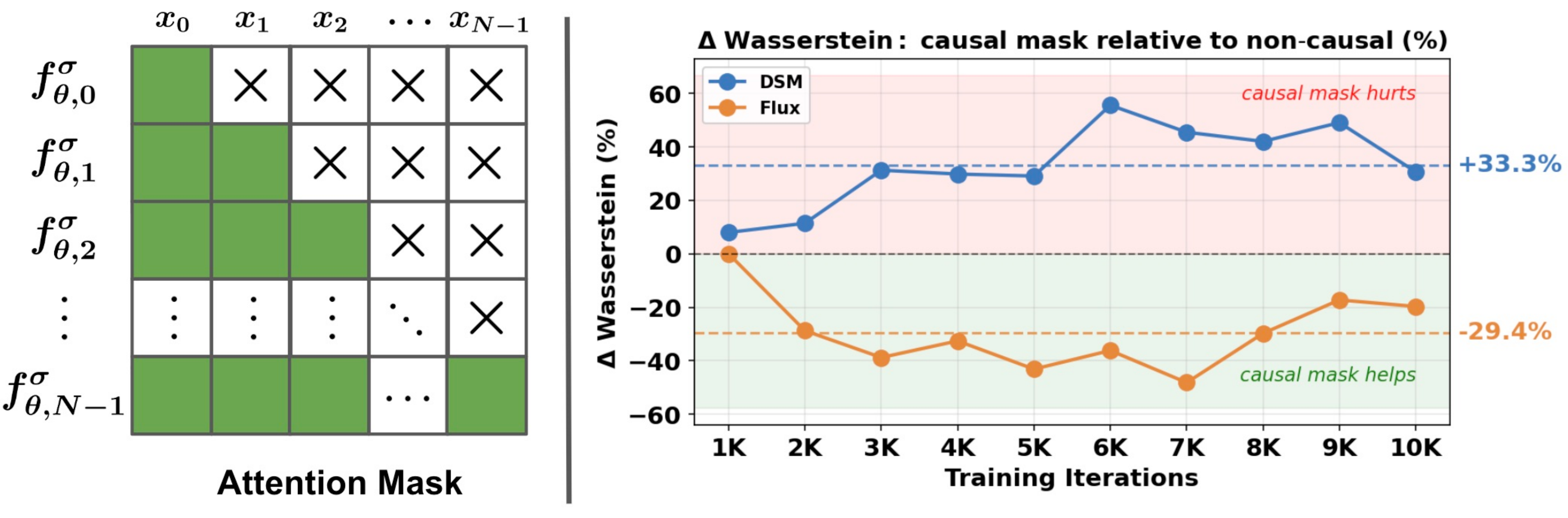}
  \caption{
  \textbf{(Left)} Causal attention mask used for trajectory generation. Rows index the output at each trajectory time, and columns index the input states at each trajectory time. The upper-triangular mask allows $f_{\theta,n}^\sigma$ to depend only on states $x_m$ with $m\le n$, enforcing autoregressive structure while still evaluating all outputs in parallel.
  \textbf{(Right)} Relative change in Wasserstein distance from adding a causal attention mask to $f_{\theta,n}^\sigma$ when training with DSM versus Flux Matching.
  }
  \label{structured_fields_figure}
  \vskip -0.1in
\end{figure}

Many generative problems have known structure among variables, and Flux
Matching lets us encode this structure directly into the architecture
representing the vector field. Score fields are gradient fields with symmetric
Jacobians (by equality of mixed partials), so directed dependencies such as
temporal autoregression are incompatible. Flux
Matching has no such constraint.

\textbf{Setup.} We simulate trajectories of two masses connected by nonlinear
springs (ODE in \Cref{app5_details_appendix}), with simulator state
$x(\tau)=(q_1(\tau),v_1(\tau),q_2(\tau),v_2(\tau))\in\mathbb{R}^4$ giving the
position $q_i$ and velocity $v_i$ of each mass at physical time $\tau$. Each data sample is a
discretized trajectory
$X=(x_0,\ldots,x_{N-1})\in\mathbb{R}^{N\times 4}$ with $x_n:=x(n\Delta\tau)$
and $N=50$. Our goal is to model the distribution over full trajectories.
Temporal order provides a natural inductive bias that later states depend
only on earlier states, which shrinks the hypothesis class and improves data
efficiency \cite{baxter2000model}. Standard diffusion samplers generate all
time points in parallel, but Flux Matching additionally lets us impose a
causal mask on $f_\theta^\sigma$, retaining the autoregressive inductive
bias \textit{without sequential sampling}. We train noise annealed DSM and Flux
Matching, each with and without a causal mask
(left of \Cref{structured_fields_figure}), on $2000$ simulated trajectories.
All four models share the same attention architecture
(\Cref{app5_details_appendix}). We evaluate via the empirical Wasserstein
distance $\mathcal{W}_2$ to the training distribution.

\textbf{Results.} The right side of \Cref{structured_fields_figure} reports
the relative change in $\mathcal{W}_2$ from adding the causal mask. The mask
consistently improves Flux Matching, confirming that the autoregressive
inductive bias helps. The same mask worsens DSM, as expected since DSM forces
the learned field to approximate a conservative score field whose symmetric
Jacobian conflicts with directed temporal dependence.
\section{Conclusion}
In this paper, we presented \textit{Flux Matching}, a new generative modeling paradigm that generalizes score matching to learn any vector field that generates samples from the target distribution. We proposed a scalable learning objective, the Flux Matching loss, together with a noise annealed extension whose learned models can be used out of the box with existing diffusion samplers and likelihood computations. We showed that Flux Matching performs well across a range of applications, including complex, high-dimensional image distributions. Most importantly, these applications demonstrate the flexibility Flux Matching gives practitioners to enforce and optimize attributes of the vector field itself. We showed that this flexibility enables faster samplers, more interpretable models, and generative models with prescribed relationships between variables.

\newpage

\begin{ack}
We thank Eric Ma, Alex Belov, Meihua Dang, Gabe Guo, Jiaqi Han, and Haotian Ye for helpful feedback and discussions. This work was supported by CZ Biohub, ONR Grant N00014-23-1-2159, the Laude Institute Moonshot Seed Grant, the Pantas And Ting Sutardja Foundation, the Wu Tsai Neurosciences Institute Big Ideas in Neuroscience Program, NIH DP2 grant 1DP2OD037052-01, and NIH K99/R00 grant 4K99HG012887-02. PPH acknowledges support from the NSF Graduate Research Fellowship.
\end{ack}

\bibliography{neurips_2026}

\newpage
\appendix
\startcontents[appendices]
\section*{Appendix Overview}
\printcontents[appendices]{l}{1}{\setcounter{tocdepth}{2}}
\newpage
\section{Proofs}
\label{proofs_appendix}

\begin{graypropbox}
\propone*
\end{graypropbox}
\begin{proof}
Follows directly from Equation 2.37 of \cite{pavliotis2014stochastic}.
\end{proof}

\begin{graypropbox}
\begin{restatable}{proposition}{proptwo}
\label{score_replacement}
Assume $p_t > 0$ on $\mathbb{R}^d$. If vector field \(f_t\) satisfies the Flux Matching condition given by \Cref{prop1} that
\begin{equation}
\label{flux_equivalent_score_condition}
\nabla \cdot (f_t(x) p_t(x)) = \nabla \cdot (\nabla \log p_t(x) p_t(x))
\end{equation}
for all $t$ and $x$, then replacing \(\nabla \log p_t\) by \(f_t\) in the sampler leaves the continuous-time marginal density evolution unchanged.
\end{restatable}
\end{graypropbox}

\begin{proof}
The Fokker--Planck contribution of replacing \(\nabla \log p_t\) by \(f_t\) differs
from the score-based one by
\[
-\Bigl( \nabla \cdot (f_t(x) p_t(x)) - \nabla \cdot (\nabla \log p_t(x) p_t(x)) \Bigr),
\]
which is zero by \eqref{flux_equivalent_score_condition}. Hence the same
\(p_t\) solves the same marginal evolution.
\end{proof}

\begin{graypropbox}
\begin{restatable}{lemma}{lemma1}
\label{lemma1}
Let
\[
g_\theta := \Pi_{\mathrm{flux}} f_\theta - \nabla \log \pdata,
\qquad
u_\theta := f_\theta - \nabla \log \pdata,
\qquad
r_\theta := \frac{1}{\pdata}\nabla\cdot(\pdata\,u_\theta).
\]
Since \(\Pi_{\mathrm{flux}} f_\theta\) and \(\nabla\log \pdata\) are gradient fields, \(g_\theta=\nabla\phi_\theta\) for some potential \(\phi_\theta\), unique up to an additive constant. Assume \(\phi_\theta,r_\theta \in L_0^2(\pdata)\) and vanishing boundary terms. Then,
\[
\mathbb{E}_{x\sim \pdata}\!\bigl[\|g_\theta(x)\|^2\bigr]
=
\int_0^\infty
\mathbb{E}_{x_0\sim \pdata, x_t|x_0}\!\left[r_\theta(x_0)\,r_\theta(x_t)\right]\,dt,
\]
where \((x_t)_{t\ge 0}\) is the stationary Langevin diffusion with generator
\[
L=\Delta+\nabla\log \pdata\cdot\nabla.
\]
\end{restatable}
\end{graypropbox}

\begin{proof}
Since both \(\Pi_{\mathrm{flux}} f_\theta\) and \(\nabla\log \pdata\) are gradient fields, so is
\[
g_\theta=\Pi_{\mathrm{flux}} f_\theta-\nabla\log \pdata.
\]
Thus \(g_\theta=\nabla\phi_\theta\) for some \(\phi_\theta\), unique up to an additive constant, which we fix by requiring \(\phi_\theta\in L_0^2(\pdata)\). Moreover, by definition of \(\Pi_{\mathrm{flux}}\),
\[
\nabla\cdot\!\bigl(\pdata\,\Pi_{\mathrm{flux}}f_\theta\bigr)=\nabla\cdot(\pdata f_\theta).
\]
Hence
\begin{align*}
r_\theta
&=
\frac{1}{\pdata}\nabla\cdot(\pdata\,u_\theta) \\
&=
\frac{1}{\pdata}\nabla\cdot\!\bigl(\pdata(f_\theta-\nabla\log \pdata)\bigr) \\
&=
\frac{1}{\pdata}\nabla\cdot\!\bigl(\pdata(\Pi_{\mathrm{flux}}f_\theta-\nabla\log \pdata)\bigr) \\
&=
\frac{1}{\pdata}\nabla\cdot(\pdata\,g_\theta) \\
&=
\nabla \cdot g_\theta + g_\theta \cdot \nabla \log \pdata.
\end{align*}
Therefore \(r_\theta\) is equivalently the Langevin Stein operator applied to \(g_\theta\), and since \(g_\theta=\nabla\phi_\theta\),
\[
r_\theta
=
\frac{1}{\pdata}\nabla\cdot(\pdata\nabla\phi_\theta)
=
\Delta\phi_\theta+\nabla\log \pdata\cdot\nabla\phi_\theta
=
L\phi_\theta.
\]
By integration by parts,
\begin{align*}
\mathbb E_{x\sim \pdata}\!\bigl[\|g_\theta(x)\|^2\bigr] &= \int \|\nabla\phi_\theta(x)\|^2 \pdata(x)\,dx \\
&= -\int \phi_\theta(x)\,L\phi_\theta(x)\,\pdata(x)\,dx \\
&= -\int \phi_\theta(x)\,r_\theta(x)\,\pdata(x)\,dx.
\end{align*}
Since \(r_\theta=L\phi_\theta\), we have \(\phi_\theta=-(-L)^{-1}r_\theta,\) and thus
\[
\mathbb E_{x\sim \pdata}\!\bigl[\|g_\theta(x)\|^2\bigr]
=
\int r_\theta(x)\,(-L)^{-1}r_\theta(x)\,\pdata(x)\,dx.
\]
We use the identity \( (-L)^{-1} = \int_0^{\infty} e^{Lt}dt \) (Ch. 11 of \cite{pavliotis2008multiscale}),
\[
\int r_\theta(x)\,(-L)^{-1}r_\theta(x)\,\pdata(x)\,dx
=
\int_0^\infty \int r_\theta(x)\,(e^{tL}r_\theta)(x)\,\pdata(x)\,dx\,dt.
\]
Finally, for the Langevin diffusion with generator \(L\),
\[
(e^{tL}r_\theta)(x)
=
\mathbb{E}\!\left[r_\theta(x_t)\mid x_0=x\right],
\]
where the expectation is over the Langevin diffusion path \((x_s)_{s\ge 0}\) started from \(x_0=x\) (see Ch. 2 of \cite{pavliotis2014stochastic}). Therefore, since \(x_0\sim \pdata\),
\begin{align*}
\int r_\theta(x)\,(e^{tL}r_\theta)(x)\,\pdata(x)\,dx
&=
\mathbb E_{x_0\sim \pdata}\!\left[
r_\theta(x_0)\,\mathbb E[r_\theta(x_t)\mid x_0]
\right] \\
&=
\mathbb E_{x_0\sim \pdata}\!\left[
\mathbb E[r_\theta(x_0)r_\theta(x_t)\mid x_0]
\right] \\
&=
\mathbb E_{x_0\sim \pdata,x_t|x_0}\!\left[r_\theta(x_0)r_\theta(x_t)\right].
\end{align*}
Hence,
\[
\int r_\theta(x)\,(-L)^{-1}r_\theta(x)\,\pdata(x)\,dx
=
\int_0^\infty \mathbb E_{x_0\sim \pdata,x_t|x_0}\!\left[r_\theta(x_0)r_\theta(x_t)\right]\,dt.
\]
\end{proof}

\begin{graypropbox}
\begin{restatable}{lemma}{lemma2}
\label{lemma2}
Let
\[
u_\theta := f_\theta - \nabla \log \pdata,
\qquad
r_\theta := \frac{1}{\pdata}\nabla\cdot(\pdata\,u_\theta).
\]
Assume \(u_\theta\) and \(e^{tL}r_\theta\) are sufficiently regular and vanishing boundary terms hold. Then, for every \(t\ge 0\),
\[
\mathbb{E}_{x_0\sim \pdata, x_t|x_0}\!\left[r_\theta(x_0)\,r_\theta(x_t)\right]
=
-\mathbb{E}_{x_0\sim \pdata}\!\left[
u_\theta(x_0)^\top \nabla (e^{tL}r_\theta)(x_0)
\right],
\]
where \((x_t)_{t\ge 0}\) is the stationary Langevin diffusion with generator
\[
L=\Delta+\nabla\log \pdata\cdot\nabla,
\]
and \(\partial x_t/\partial x_0\) denotes the Jacobian of the associated Langevin path. Equivalently,
\[
\mathbb{E}_{x_0\sim \pdata, x_t|x_0}\!\left[r_\theta(x_0)\,r_\theta(x_t)\right]
=
-\mathbb{E}_{x_0\sim \pdata, x_t|x_0}\!\left[
u_\theta(x_0)^\top
\partial x_t/\partial x_0^{\top}\nabla_{x_t} r_\theta(x_t)
\right].
\]
\end{restatable}
\end{graypropbox}

\begin{proof}
For the Langevin diffusion with generator \(L\),
\[
(e^{tL}r_\theta)(x)
=
\mathbb{E}\!\left[r_\theta(x_t)\mid x_0=x\right],
\]
where the expectation is over diffusion paths \((x_s)_{s\ge 0}\) started from \(x_0=x\) (Ch. 2 of \cite{pavliotis2014stochastic}). Hence,
\begin{align*}
\mathbb{E}_{x_0\sim \pdata, x_t|x_0}\!\left[r_\theta(x_0)\,r_\theta(x_t)\right]
&=
\mathbb{E}_{x_0\sim \pdata}\!\left[
r_\theta(x_0)\,\mathbb{E}[r_\theta(x_t)\mid x_0]
\right] \\
&=
\mathbb{E}_{x_0\sim \pdata}\!\left[
r_\theta(x_0)\,(e^{tL}r_\theta)(x_0)
\right].
\end{align*}
Using \(r_\theta = \pdata^{-1}\nabla\cdot(\pdata\,u_\theta)\), we have
\begin{align*}
\mathbb{E}_{x_0\sim \pdata}\!\left[
r_\theta(x_0)\,(e^{tL}r_\theta)(x_0)
\right]
&=
\int \nabla\cdot\!\bigl(\pdata(x)u_\theta(x)\bigr)\,(e^{tL}r_\theta)(x)\,dx \\
&=
-\int \pdata(x)\,u_\theta(x)^\top \nabla (e^{tL}r_\theta)(x)\,dx
\qquad
\text{(integration by parts)} \\
&= 
-\mathbb{E}_{x_0\sim \pdata}\!\left[
u_\theta(x_0)^\top \nabla (e^{tL}r_\theta)(x_0)
\right].
\end{align*}
Finally, let \(\partial x_t/\partial x_0\) denote the Jacobian of the path generated by the Langevin diffusion. Since
\[
\nabla (e^{tL}r_\theta)(x_0)
=
\nabla_{x_0}\mathbb{E}\!\left[r_\theta(x_t)\mid x_0\right]
=
\mathbb{E}\!\left[
\partial x_t/\partial x_0^{\top}\nabla_{x_t}r_\theta(x_t)
\mid x_0
\right],
\]
we obtain
\[
\mathbb{E}_{x_0\sim \pdata, x_t|x_0}\!\left[r_\theta(x_0)\,r_\theta(x_t)\right]
=
-\mathbb{E}_{x_0\sim \pdata, x_t|x_0}\!\left[
u_\theta(x_0)^\top
\partial x_t/\partial x_0^{\top}\nabla_{x_t}r_\theta(x_t)
\right].
\]
\end{proof}

\begin{graypropbox}
\begin{restatable}{lemma}{lemma3}
\label{lemma3}
Let \(q\) be any strictly positive probability density on \([0,\infty)\), let
\(\partial x_t/\partial x_0\), and define
\[
\widetilde{\mathcal L}_{\mathrm{Flux}}(\theta)
:=
-\mathbb{E}_{\substack{t\sim q\\x_0\sim \pdata,\,x_t|x_0}}\!\left[
q(t)^{-1}\,
u_\theta(x_0)^\top \partial x_t/\partial x_0^{\top}\nabla_{x_t} r_\theta(x_t)
\right],
\]
\[
\mathcal L_{\mathrm{Flux}}(\theta)
:=
-\mathbb{E}_{\substack{t\sim q\\x_0\sim \pdata,\,x_t|x_0}}\!\left[
q(t)^{-1}\,
u_\theta(x_0)^\top
\operatorname{sg}\!\bigl(\partial x_t/\partial x_0^{\top}\nabla_{x_t} r_\theta(x_t)\bigr)
\right].
\]
Assume differentiation and expectation may be interchanged, and that the stationary Langevin diffusion is reversible with respect to \(\pdata\). Then,
\[
\nabla_\theta \widetilde{\mathcal L}_{\mathrm{Flux}}(\theta)
=
2\,\nabla_\theta \mathcal L_{\mathrm{Flux}}(\theta).
\]
\end{restatable}
\end{graypropbox}

\begin{proof}
We differentiate componentwise in \(\theta\). For any parameter \(\theta_j\),
\begin{align*}
\partial_{\theta_j}\widetilde{\mathcal L}_{\mathrm{Flux}}(\theta)
&=
-\mathbb{E}\!\left[
q(t)^{-1}\,
(\partial_{\theta_j}u_\theta(x_0))^\top
\partial x_t/\partial x_0^{\top}\nabla_{x_t} r_\theta(x_t)
\right] \\
&\qquad
-\mathbb{E}\!\left[
q(t)^{-1}\,
u_\theta(x_0)^\top
\partial x_t/\partial x_0^{\top}\nabla_{x_t}(\partial_{\theta_j}r_\theta(x_t))
\right].
\end{align*}
Since \(r_\theta=\frac1{\pdata}\nabla\cdot(\pdata\,u_\theta)\) depends linearly on \(u_\theta\), we have
\[
\partial_{\theta_j}r_\theta
=
\frac1{\pdata}\nabla\cdot\!\bigl(\pdata\,\partial_{\theta_j}u_\theta\bigr).
\]
Applying Lemma~2 with \(u_\theta\) replaced by \(\partial_{\theta_j}u_\theta\) gives
\[
-\mathbb{E}\!\left[
q(t)^{-1}\,
(\partial_{\theta_j}u_\theta(x_0))^\top
\partial x_t/\partial x_0^{\top}\nabla_{x_t} r_\theta(x_t)
\right]
=
\mathbb{E}\!\left[
q(t)^{-1}\,
(\partial_{\theta_j}r_\theta(x_0))\,r_\theta(x_t)
\right].
\]
Similarly,
\[
-\mathbb{E}\!\left[
q(t)^{-1}\,
u_\theta(x_0)^\top
\partial x_t/\partial x_0^{\top}\nabla_{x_t}(\partial_{\theta_j}r_\theta(x_t))
\right]
=
\mathbb{E}\!\left[
q(t)^{-1}\,
r_\theta(x_0)\,\partial_{\theta_j}r_\theta(x_t)
\right].
\]
By reversibility of the stationary Langevin diffusion, $\mathbb E[\varphi(x_0,x_t)] = \mathbb E[\varphi(x_t,x_0)]$
for any test function \(\varphi\), so the two right-hand sides are equal. Thus,
\[
\partial_{\theta_j}\widetilde{\mathcal L}_{\mathrm{Flux}}(\theta)
=
-2\,\mathbb{E}\!\left[
q(t)^{-1}\,
(\partial_{\theta_j}u_\theta(x_0))^\top
\partial x_t/\partial x_0^{\top}\nabla_{x_t} r_\theta(x_t)
\right].
\]
On the other hand, by definition of stop-gradient,
\begin{align*}
\partial_{\theta_j}\mathcal L_{\mathrm{Flux}}(\theta)
&=
-\mathbb{E}\!\left[
q(t)^{-1}\,
(\partial_{\theta_j}u_\theta(x_0))^\top
\operatorname{sg}\!\bigl(\partial x_t/\partial x_0^{\top}\nabla_{x_t} r_\theta(x_t)\bigr)
\right]\\
&=-\mathbb{E}\!\left[
q(t)^{-1}\,
(\partial_{\theta_j}u_\theta(x_0))^\top
\partial x_t/\partial x_0^{\top}\nabla_{x_t} r_\theta(x_t)
\right].
\end{align*}
Therefore
\[
\partial_{\theta_j}\widetilde{\mathcal L}_{\mathrm{Flux}}(\theta)
=
2\,\partial_{\theta_j}\mathcal L_{\mathrm{Flux}}(\theta)
\qquad
\text{for every }j,
\]
and
\[
\nabla_\theta \widetilde{\mathcal L}_{\mathrm{Flux}}(\theta)
=
2\,\nabla_\theta \mathcal L_{\mathrm{Flux}}(\theta).
\]
\end{proof}

\begin{graypropbox}
\theoremone*
\end{graypropbox}
\begin{proof}
We have
\begin{alignat*}{2}
\mathcal J(\theta)
&=
\mathbb{E}_{x_0\sim \pdata}\!\left[
\|\Pi_{\mathrm{flux}}f_\theta(x_0)-\nabla\log \pdata(x_0)\|^2
\right]
&\qquad& \\
&=
\int_0^\infty
\mathbb{E}_{x_0\sim \pdata,\,x_t|x_0}\!\left[
r_\theta(x_0)\,r_\theta(x_t)
\right]\,dt
&& \text{(\Cref{lemma1})} \\
&=
-\int_0^\infty
\mathbb{E}_{x_0\sim \pdata,\,x_t|x_0}\!\left[
u_\theta(x_0)^\top
\partial x_t/\partial x_0^{\top}\nabla_{x_t} r_\theta(x_t)
\right]\,dt
&& \text{(\Cref{lemma2}).}
\end{alignat*}
Differentiating with respect to \(\theta\) and applying \Cref{lemma3} yields
\[
\nabla_\theta \widetilde{\mathcal{J}}(\theta)
=
-2\,\nabla_\theta
\mathbb{E}_{\substack{t\sim q\\x_0\sim \pdata,\,x_t|x_0}}\!\left[
q(t)^{-1}\,
u_\theta(x_0)^\top
\operatorname{sg}\!\bigl(
\partial x_t/\partial x_0^{\top}\nabla_{x_t} r_\theta(x_t)
\bigr)
\right]
=
2\,\nabla_\theta \mathcal L_{\mathrm{Flux}}(\theta).
\]
\end{proof}

\section{Experiment Details}
\label{appendix_experiment_section}
\subsection{Application 1: \textit{Controllable} Generative Fields}
\label{app1_details_appendix}

\paragraph{Goal.} In this toy setting of a two-dimensional Gaussian mixture, we aim to show that (1) Flux Matching assign $0$ loss to fields that preserve the target distribution but have very different dynamics, while penalizing fields that change the target distribution and (2) score matching assigns high loss to both fields that preserve the distribution (but have non score dynamics) and fields that change the target distribution. Ultimately, the core message of this experiment is that Flux Matching enables control over interesting attributes about the vector field (that still preserve the distribution) while score matching does not. 

\paragraph{Target distribution.}
We use
\begin{equation}
p_\sigma(x)
=
\sum_{k=1}^3
\pi_k\,\mathcal N(x;\mu_k,\nu_\sigma^2 I),
\end{equation}
with
\begin{equation}
\mu_1=(0,\,2.3),\quad
\mu_2=(-1.99,\,-1.15),\quad
\mu_3=(1.99,\,-1.15),\quad
\nu_\sigma^2=\sigma_0^2+\sigma^2=0.625.
\end{equation}
The three modes lie at the vertices of an equilateral triangle. Since \(p_\sigma\) is known analytically, its score is also available in closed form:
\begin{equation}
\label{closed_score_form_toy_experiment}
s_\sigma(x)
=
\nabla \log p_\sigma(x)
=
\frac{m_\sigma(x)-x}{\nu_\sigma^2},
\qquad
m_\sigma(x)
=
\sum_{k=1}^3
\omega_k^\sigma(x)\mu_k,
\end{equation}
where
\begin{equation}
\omega_k^\sigma(x)
=
\frac{
\pi_k\exp\!\left(-\|x-\mu_k\|^2/(2\nu_\sigma^2)\right)
}{
\sum_{\ell=1}^3
\pi_\ell\exp\!\left(-\|x-\mu_\ell\|^2/(2\nu_\sigma^2)\right)
}.
\end{equation}

\paragraph{Vector field families.}
All fields are perturbations of the score,
\begin{equation}
f_\theta(x)=s_\sigma(x)+u_\theta(x),
\qquad
u_\theta(x)
=
\theta_0 c_{\mathrm{rot}}(x)
+
\theta_1 c_{\mathrm{tri}}(x)
+
\theta_2 c_{\mathrm{skew}}(x).
\end{equation}
The three controls (aka attributes) respectively induce global rotation, clockwise circulation around the triangle, and localized Jacobian asymmetry. We will then calculate specific metrics on these vector fields (specified under ``Metrics''), which give us \Cref{alternative_vector_fields_figure} and \Cref{app1_figure}. 

For the distribution-preserving family, we use
\begin{equation}
c_{\mathrm{rot}}(x)\propto J s_\sigma(x),
\qquad
c_{\mathrm{tri}}(x)\propto -\frac{J\nabla\psi_{\mathrm{tri}}(x)}{p_\sigma(x)},
\qquad
c_{\mathrm{skew}}(x)\propto \frac{J\nabla\psi_{\mathrm{skew}}(x)}{p_\sigma(x)},
\end{equation}
where \(J=\begin{psmallmatrix}0&-1\\1&0\end{psmallmatrix}\). The exact forms of \(\psi_{\mathrm{tri}}\) and \(\psi_{\mathrm{skew}}\) are not essential to this experiment; they are smooth templates chosen to produce visually interesting vector fields for \Cref{alternative_vector_fields_figure}. The key structural requirement is that each distribution preserving attribute has the form \(J\nabla\psi/p_\sigma\), which guarantees
\(\nabla\cdot(p_\sigma c)=0\). Within this constraint, we tuned the template parameters to make the induced vector fields visually clear, high-contrast, and qualitatively distinct, where \(\psi_{\mathrm{tri}}\) produces circulation along the triangle and \(\psi_{\mathrm{skew}}\) produces a localized asymmetric perturbation. As a result, the somewhat elaborate formulas below should be viewed as implementation choices for constructing illustrative vector fields with interesting attributes, not as something rigorously motivated.

The scalar function \(\psi_{\mathrm{tri}}\) concentrates near the triangle edges, while \(\psi_{\mathrm{skew}}\) is a localized anisotropic bump. More precisely, let \(v_j\) be the triangle vertices, \(e_j\) the unit direction of edge \(j\), \(n_j\) the outward normal to edge \(j\), and \(\ell_j\) the edge length. For each edge, define the along-edge coordinate and normal distance
\begin{equation}
a_j(x)=e_j^\top(x-v_j),
\qquad
d_j(x)=n_j^\top(x-v_j).
\end{equation}
We define an edge-localized template by
\begin{equation}
R_j(x)
=
\exp\!\left(-\frac{d_j(x)^2}{2w^2}\right)
\operatorname{sigmoid}\!\left(\frac{\kappa a_j(x)}{w}\right)
\operatorname{sigmoid}\!\left(\frac{\kappa(\ell_j-a_j(x))}{w}\right),
\end{equation}
which is large near edge \(j\) and small away from that edge segment. We also define a smooth triangle level function
\begin{equation}
h(x)
=
\frac{1}{\beta}
\log
\sum_{j=1}^3
\exp\!\left(\beta(n_j^\top x-n_j^\top v_j)\right).
\end{equation}
Then,
\begin{equation}
\psi_{\mathrm{tri}}(x)
=
\exp\!\left(
-\frac{\operatorname{ReLU}(h(x))^2}{2\eta^2}
\right)
\left[
\frac{1}{3}\sum_{j=1}^3 R_j(x)
+
\lambda_{\mathrm{int}}\operatorname{sigmoid}(-\beta h(x))
\right].
\end{equation}
In the experiment, we use \(w=0.24\), \(\kappa=5.0\), \(\beta=10.0\),
\(\lambda_{\mathrm{int}}=0.12\), and \(\eta=0.9\).

For the skew template, let
\begin{equation}
d=(\cos\alpha,\sin\alpha),
\qquad
z=\rho R d,
\qquad
d_\perp=Jd,
\end{equation}
where \(R=2.3\), \(\alpha=-20^\circ\), and \(\rho=0.78\). For \(r=x-z\), define
\begin{equation}
r_\parallel=d^\top r,
\qquad
r_\perp=d_\perp^\top r.
\end{equation}
We set
\begin{equation}
\psi_{\mathrm{skew}}(x)
=
\exp\!\left[
-\frac12
\left(
\frac{r_\parallel^2}{a^2}
+
\frac{r_\perp^2}{b^2}
\right)
\right]
\left[
1+
\lambda_{\mathrm{bias}}
\operatorname{sigmoid}
\!\left(
\frac{\gamma d^\top x}{R}
\right)
\right],
\end{equation}
with \(a=0.60\), \(b=1.05\), \(\lambda_{\mathrm{bias}}=0.45\), and
\(\gamma=1.75\).

Each attribute vector field is constructed to be of \(0\) flux divergence (aka should make the Flux Matching loss \(0\)):
\begin{equation}
\nabla\cdot\bigl(p_\sigma(x)c(x)\bigr)=0.
\end{equation}
Hence every linear combination satisfies
\begin{equation}
\nabla\cdot\bigl(p_\sigma(x)u_\theta(x)\bigr)=0,
\end{equation}
so these perturbations change the vector field dynamics without changing the stationary distribution \(p_\sigma\). For comparison, we also construct a distribution-violating family by removing the \(90^\circ\) rotation from the same templates:
\begin{equation}
\tilde c_{\mathrm{rot}}(x)=s_\sigma(x),
\qquad
\tilde c_{\mathrm{tri}}(x)\propto \frac{\nabla\psi_{\mathrm{tri}}(x)}{p_\sigma(x)},
\qquad
\tilde c_{\mathrm{skew}}(x)\propto \frac{\nabla\psi_{\mathrm{skew}}(x)}{p_\sigma(x)}.
\end{equation}
Then
\begin{equation}
u_\theta^{\mathrm{viol}}(x)
=
\theta_0\tilde c_{\mathrm{rot}}(x)
+
\theta_1\tilde c_{\mathrm{tri}}(x)
+
\theta_2\tilde c_{\mathrm{skew}}(x),
\qquad
\nabla\cdot\bigl(p_\sigma u_\theta^{\mathrm{viol}}\bigr)\neq 0
\end{equation}
in general, so these perturbations change the stationary distribution.

Finally, the distribution preserving attributes are rescaled to have unit RMS norm under \(p_\sigma\):
\begin{equation}
c(x)
\leftarrow
\frac{c(x)}
{
\sqrt{\mathbb E_{x\sim p_\sigma}\|c(x)\|_2^2}
}.
\end{equation}
The scale factors are estimated using \(1200\) samples from \(p_\sigma\) and are reused for the corresponding distribution-violating fields, making the coefficients
\(\theta_0,\theta_1,\theta_2\) comparable across control directions.

\paragraph{Evaluation grid.}
All metrics are computed on a \(25\times25\) grid over \([-5.5,5.5]^2\) and let \(\{x_i\}_{i=1}^N\) be the grid points. We use normalized density weights
\begin{equation}
w_i
=
\frac{p_\sigma(x_i)}
{\sum_{j=1}^N p_\sigma(x_j)}.
\end{equation}

\paragraph{Metrics.}
All metrics use the \(25\times25\) grid and density weights \(w_i\propto p_\sigma(x_i)\).

\emph{Mixing speed.}
For observables
\begin{equation}
\Phi=\{x_1,\;x_2,\;\|x\|_2,\;\cos(\angle x),\;\sin(\angle x)\},
\end{equation}
we discretize \(\mathcal L_{f_\theta}=\Delta+f_\theta\cdot\nabla\) and, for each
\(\varphi\in\Phi\), solve the mean-constrained Poisson equation
\begin{equation}
-\mathcal L_{f_\theta}\psi_\varphi
=
\varphi-\mathbb E_w[\varphi].
\end{equation}
We estimate the integrated autocorrelation time by
\begin{equation}
\tau_\varphi(f_\theta)
=
\frac{
2\langle \varphi-\mathbb E_w[\varphi],\,\psi_\varphi\rangle_w
}{
\operatorname{Var}_w(\varphi)
},
\qquad
M_{\mathrm{mix}}(f_\theta)
=
\left(
\frac{1}{|\Phi|}
\sum_{\varphi\in\Phi}\tau_\varphi(f_\theta)
\right)^{-1}.
\end{equation}
Larger \(M_{\mathrm{mix}}\) means faster mixing.

\emph{Triangle shape.}
We measure cosine alignment between the perturbation and the triangle-shape control:
\begin{equation}
M_{\mathrm{tri}}(f_\theta)
=
\frac{
\langle u_\theta,c_{\mathrm{tri}}\rangle_w
}{
\|u_\theta\|_w\,\|c_{\mathrm{tri}}\|_w
}.
\end{equation}

\emph{Jacobian skewness.}
For \(u_\theta=(u_1,u_2)\), we measure the weighted squared antisymmetric Jacobian component:
\begin{equation}
M_{\mathrm{skew}}(f_\theta)
=
\sum_i w_i\,
\frac12
\left[
(\partial_{x_2}u_1)(x_i)
-
(\partial_{x_1}u_2)(x_i)
\right]^2.
\end{equation}

\emph{Distribution violation.}
For the non-preserving family, we measure the stationarity residual
\begin{equation}
r_\theta(x)
=
\nabla\cdot u_\theta^{\mathrm{viol}}(x)
+
u_\theta^{\mathrm{viol}}(x)\cdot s_\sigma(x),
\qquad
V(f_\theta)
=
\left(
\sum_i w_i r_\theta(x_i)^2
\right)^{1/2}.
\end{equation}
Here \(V(f_\theta)=0\) exactly when the perturbation preserves \(p_\sigma\).

\paragraph{Losses.}
The score matching loss is
\begin{equation}
L_{\mathrm{SM}}(f_\theta)
=
\frac12
\mathbb E_{x\sim p_\sigma}
\|f_\theta(x)-s_\sigma(x)\|_2^2
=
\frac12
\mathbb E_{x\sim p_\sigma}
\|u_\theta(x)\|_2^2.
\end{equation}
Note again that $s_\sigma(x)$ is given to us in closed form via \Cref{closed_score_form_toy_experiment}. In the experiment, we approximate this expectation on the same
\(25\times25\) grid used for the metrics. 

The Flux Matching loss is estimated separately by Monte Carlo, not
on the \(25\times25\) grid (since the Flux Matching loss requires short MCMC steps). For each \(\theta\), we sample \(B=256\) samples
\(x_0^{(i)}\sim p_\sigma\) as a minibatch, and compute the Flux Matching loss via \Cref{flux_matching_loss}. We set the horizon sampler distributed to be the truncated uniform, $q =\mathcal{U}[0,T]$. We average this estimate over \(64\) independent minibatches.

\paragraph{Sweep and plotting.}
For \Cref{app1_figure}, we sweep
\begin{equation}
\theta_0,\theta_1,\theta_2\in\{-2.5,0,2.5\},
\end{equation}
giving \(27\) fields. For each field, we record the relevant metric together with the score matching and Flux Matching losses. Since the two losses have different raw scales, each loss is divided by its mean value over the distribution-violating family before plotting. Curves are produced by binning the horizontal axis into \(12\) equally spaced bins and plotting the mean and standard error in each bin.

For \Cref{alternative_vector_fields_figure}, we perform a separate one-dimensional scan along each control axis using \(33\) values in \([-8,8]\). We display the field with the best value of each metric. For the Jacobian-skewness panel, if multiple fields are within \(99.5\%\) of the maximum skewness, we choose the one with the smallest triangle-flow alignment so that the displayed field isolates skewness rather than triangle circulation. The background empirical samples (in red dots) are \(1800\) draws from \(p_\sigma\).

\subsection{Application 2: \textit{Interpretable} Generative Fields}
\label{app2_details_appendix}

\paragraph{Goal.}
The goal of this experiment is to show that Flux Matching can learn interpretable biological vector fields, such as RNA velocity. We use the scVelo dynamical model of \cite{bergen2020generalizing}, but fit its parameters with Flux Matching rather than their EM-style latent-time optimization.

\paragraph{Background on scVelo's Dynamical Model \cite{bergen2020generalizing}.}
The Dynamical model generalizes RNA velocity beyond the original steady-state model
\cite{la2018rna} by fitting a likelihood-based dynamical model of splicing
kinetics. The scVelo dynamical model describes each gene \(g\)'s unspliced and spliced abundances by
\[
\frac{d u_g(\tau)}{d\tau} = \alpha_g(\tau) - \beta_g u_g(\tau),
\qquad
\frac{d s_g(\tau)}{d\tau} = \beta_g u_g(\tau) - \gamma_g s_g(\tau),
\]
where the transcription rate switches between phases \(k \in \{\textsc{Induction}, \textsc{Repression}\}\):
\[
\alpha_g(\tau) =
\begin{cases}
\alpha_g, & k = \textsc{Induction}\ (\tau \leq \tau_s^g),\\
0, & k = \textsc{Repression}\ (\tau > \tau_s^g).
\end{cases}
\]
Induction starts from \((u_g, s_g) = (0, 0)\) at \(\tau = 0\); repression starts from \((u_g(\tau_s^g), s_g(\tau_s^g))\). Closed-form solutions \(\bar u_g(\tau, k; \Theta_g)\) and \(\bar s_g(\tau, k; \Theta_g)\) yield a Gaussian observation model
\[
\log p\bigl(u_{ig}, s_{ig} \mid k, \tau, \Theta_g\bigr) = \log \mathcal{N}\bigl(u_{ig};\, \bar u_g(\tau, k),\, \sigma_{u,g}^2\bigr) + \log \mathcal{N}\bigl(s_{ig};\, \bar s_g(\tau, k),\, \sigma_{s,g}^2\bigr),
\]
with per-gene parameters \(\Theta_g = (\alpha_g, \beta_g, \gamma_g, \tau_s^g, \sigma_{u,g}, \sigma_{s,g})\). The fitting procedure is summarized in Algorithm~\ref{alg:scvelo_em}.
\begin{algorithm}[t]
\caption{scVelo dynamical model fitting \cite{bergen2020generalizing}}
\label{alg:scvelo_em}
\begin{algorithmic}[1]
\Require Unspliced and spliced counts \(\{(u_{ig}, s_{ig})\}_{i,g}\)
\For{each gene \(g\) (independently)}
    \State Initialize \(\Theta_g\).
    \Repeat
        \State \textbf{E-step.} For each cell \(i\):
        \State \quad For each phase \(k\), find the best latent time \(\tau_{ig}^{(k)} = \text{argmax}_\tau \log p(u_{ig}, s_{ig} \mid k, \tau, \Theta_g)\).
        \State \quad Pick the better phase: \(k_{ig} = \text{argmax}_k \log p(u_{ig}, s_{ig} \mid k, \tau_{ig}^{(k)}, \Theta_g)\), and set \(\tau_{ig} = \tau_{ig}^{(k_{ig})}\).
        \State \textbf{M-step.} Update \(\Theta_g \leftarrow \text{argmax}_{\Theta_g} \sum_i \log p(u_{ig}, s_{ig} \mid k_{ig}, \tau_{ig}, \Theta_g)\).
    \Until{convergence}
\EndFor
\end{algorithmic}
\end{algorithm}

Interestingly, Flux Matching removes the need for these EM-style updates. Instead of
alternating between assigning latent times under the current kinetic parameters
and updating the parameters under those assignments, we directly optimize the
parameters of the vector field using the Flux Matching loss.
\paragraph{Datasets.}
We use five RNA-velocity datasets: Bone marrow, Dentate Gyrus, Gastrulation, Hindbrain, and Pancreas. The datasets are loaded as follows
\begin{verbatim}
import scvelo as scv

datasets = {
    "bone_marrow": scv.datasets.bonemarrow(),
    "dentate_gyrus": scv.datasets.dentategyrus(),
    "gastrulation": scv.datasets.gastrulation(),
    "pancreas": scv.datasets.pancreas(),
    "hindbrain": scv.read("path/to/hindbrain.h5ad"),
}
\end{verbatim}
We provide further instructions on obtaining and loading the hindbrain dataset in our code.

\paragraph{Data preprocessing.}
We apply the same preprocessing pipeline to each dataset
\begin{verbatim}
import scanpy as sc
import scvelo as scv

sc.pp.filter_cells(adata, min_counts=1)
scv.pp.filter_and_normalize(adata, min_shared_counts=20)
sc.pp.highly_variable_genes(
    adata,
    n_top_genes=2000,
    flavor="seurat",
    subset=True,
)
sc.pp.pca(adata)
sc.pp.neighbors(adata, n_pcs=30, n_neighbors=30)
scv.pp.moments(adata, n_neighbors=None, n_pcs=None)
\end{verbatim}

\paragraph{Model parameterization.}
The scVelo dynamical model uses gene-specific splicing and degradation dynamics that depend on a discrete latent transcriptional state (the $4$ phases, namely induction initiation, induction saturation, repression initiation, and repression decay, details can be found in \cite{bergen2020generalizing}). Because we optimize using gradient descent, we must use a continuous analogue to the discrete assignment:
\begin{equation}
\label{eq:rna_velocity_vector_field}
\begin{aligned}
\alpha_g(u_g,s_g)
&=
\mathrm{softplus}
\left(
a_g u_g + b_g s_g + c_g
\right),
\\
\frac{d u_g}{d\tau}
&=
\alpha_g(u_g,s_g) - \beta_g u_g,
\\
\frac{d s_g}{d\tau}
&=
\beta_g u_g - \gamma_g s_g .
\end{aligned}
\end{equation}
Here \(\alpha_g(u_g,s_g)>0\) is the learned transcription rate, approximated by a first-order approximation $
\alpha_g(u_g,s_g)
=
\mathrm{softplus}\!\left(
a_g u_g + b_g s_g + c_g
\right),
$ while $\beta, \gamma \geq 0$ are the learned splicing and degradation rates, respectively. The
optimized gene-specific parameters are
\(\tilde{\beta}_g,\tilde{\gamma}_g,a_g,b_g,c_g\). We enforce positivity of the
kinetic rates by setting
\begin{equation}
\label{eq:rna_velocity_positive_rates}
\beta_g = \mathrm{softplus}(\tilde{\beta}_g),
\qquad
\gamma_g = \mathrm{softplus}(\tilde{\gamma}_g).
\end{equation}
The coefficients \(a_g,b_g,c_g\) provide a continuous analogue of
the latent transcriptional-state assignment in the dynamical model. Rather than
switching among discrete transcription rates, the model learns a smooth
state-dependent transcription rate. We initialize \(a_g,b_g,c_g\) at zero, so
the transcription rate is initially constant and the model begins as a simple
linear kinetic ODE.

\paragraph{Training details.}
The RNA velocity experiments on all datasets use the same Flux Matching hyperparameters, summarized in \Cref{tab:app2_hyperparameters}. For the dynamical model, we use the package's default hyperparameters.

\begin{table}[t]
\centering
\caption{Hyperparameters for the RNA-velocity experiment.}
\label{tab:app2_hyperparameters}
\begin{tabular}{lc}
\toprule
Hyperparameter & Value \\
\midrule
Training iterations & \(100\) \\
Optimizer & Adam \\
Learning rate & \(10^{-3}\) \\
$\sigma^2$ & \(10^{-3}\) \\
Horizon distribution & \(q(t)=\mathcal{U}[0,T]\) \\
\bottomrule
\end{tabular}
\end{table}

\paragraph{Visualizing Learned RNA Velocity.}
We provide visualizations of the inferred RNA velocity and ground truth progression of the bone marrow dataset in \Cref{rnavelo_bonemarrow_figure}, the dentate gyrus dataset in \Cref{rnavelo_dentategyrus_figure}, the gastrulation dataset in \Cref{rnavelo_gastrulation_figure}, the hindbrain dataset in \Cref{rnavelo_hindbrain_figure}, and the pancreas dataset in \Cref{rnavelo_pancreas_figure}. 
\begin{figure}
  \centering
  \includegraphics[width=0.9\textwidth]{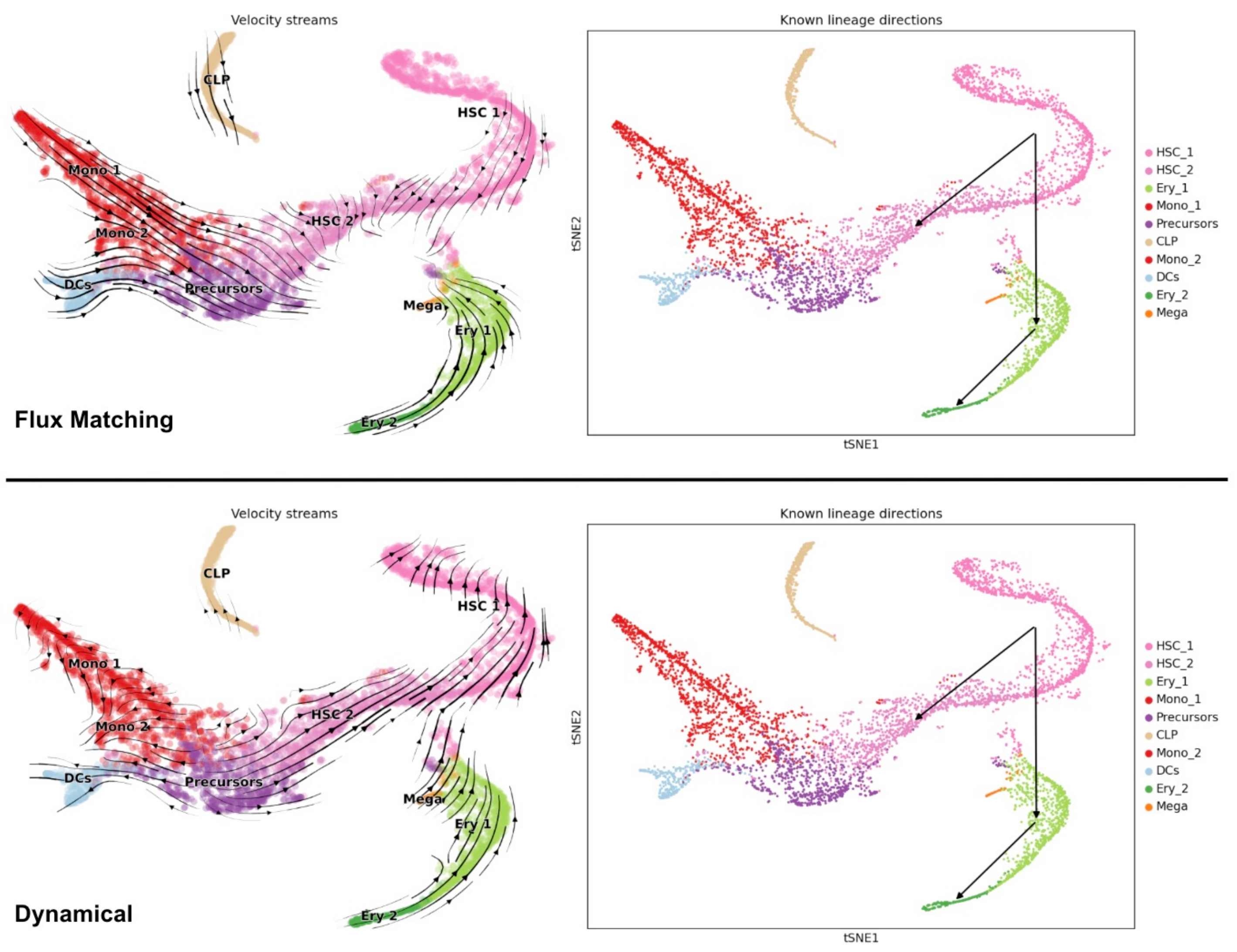}
  \caption{ \textbf{Bone Marrow Dataset}. (Left half) inferred RNA velocity (Right half) ground truth biological progression given by arrows}
  \label{rnavelo_bonemarrow_figure}
\end{figure}
\begin{figure}
  \centering
  \includegraphics[width=0.9\textwidth]{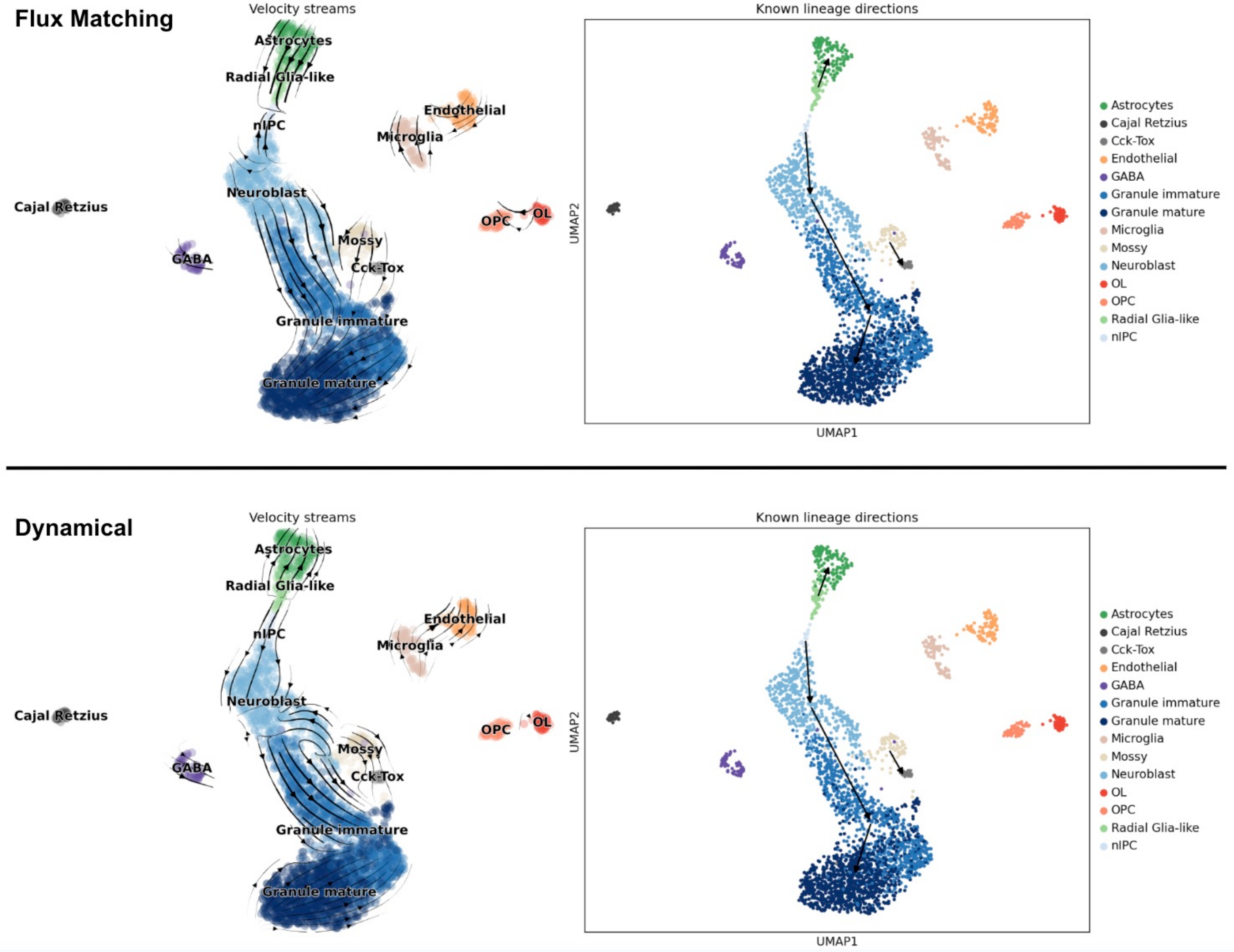}
  \caption{ \textbf{Dentate Gyrus Dataset}. (Left half) inferred RNA velocity (Right half) ground truth biological progression given by arrows}
  \label{rnavelo_dentategyrus_figure}
\end{figure}
\begin{figure}
  \centering
  \includegraphics[width=0.9\textwidth]{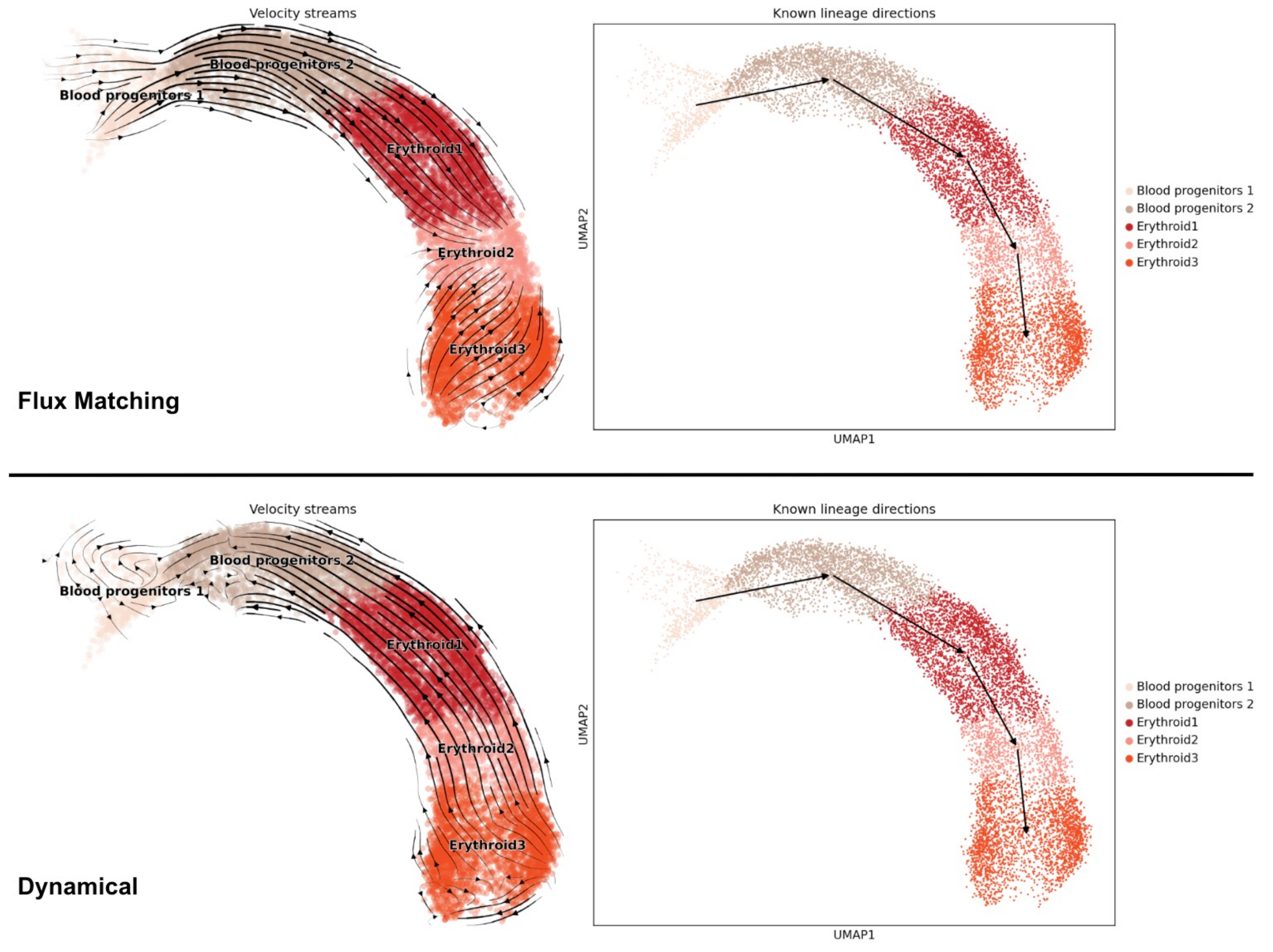}
  \caption{ \textbf{Gastrulation Dataset}. (Left half) inferred RNA velocity (Right half) ground truth biological progression given by arrows}
  \label{rnavelo_gastrulation_figure}
\end{figure}
\begin{figure}
  \centering
  \includegraphics[width=0.9\textwidth]{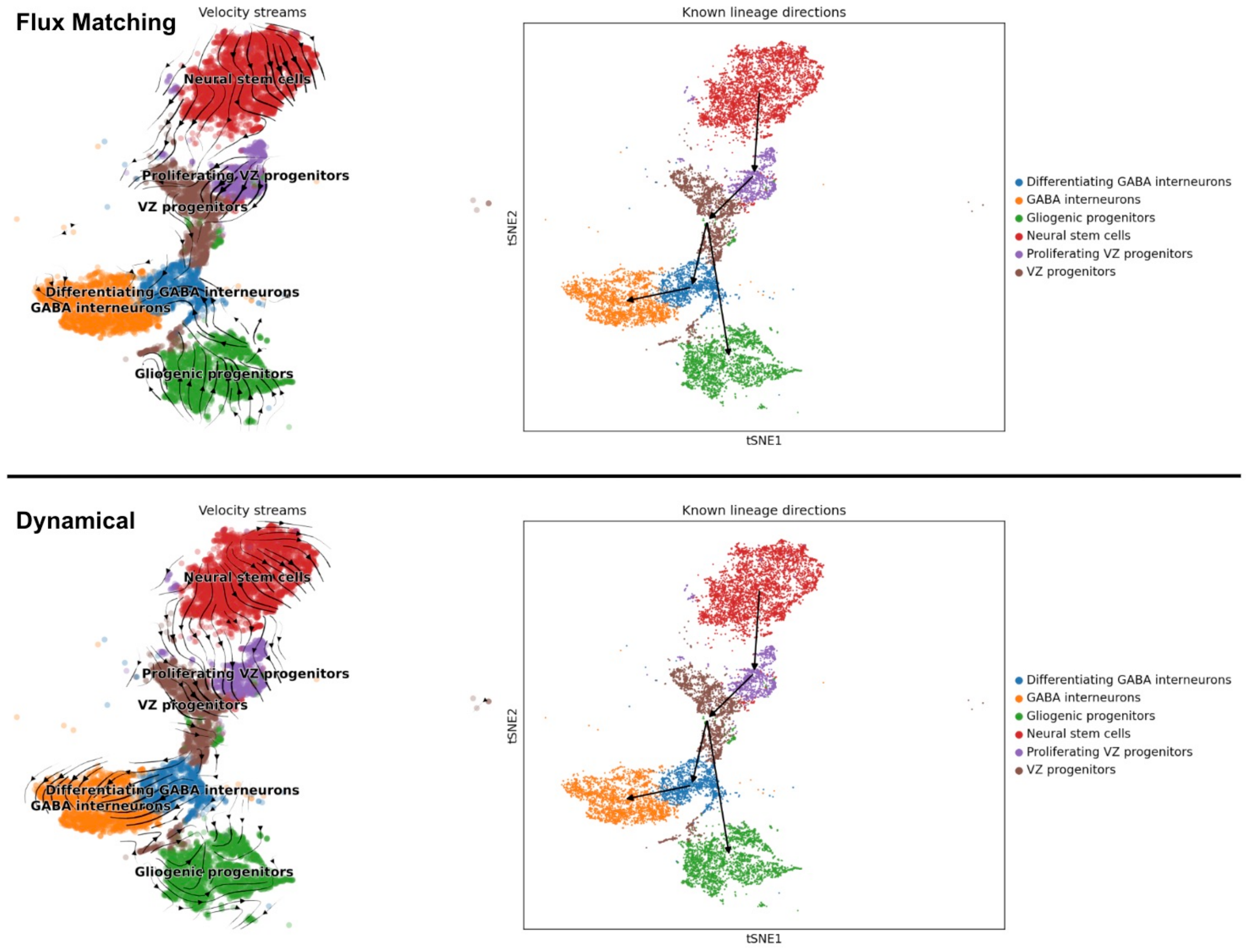}
  \caption{ \textbf{Hindbrain Dataset}. (Left half) inferred RNA velocity (Right half) ground truth biological progression given by arrows}
  \label{rnavelo_hindbrain_figure}
\end{figure}
\begin{figure}
  \centering
  \includegraphics[width=0.9\textwidth]{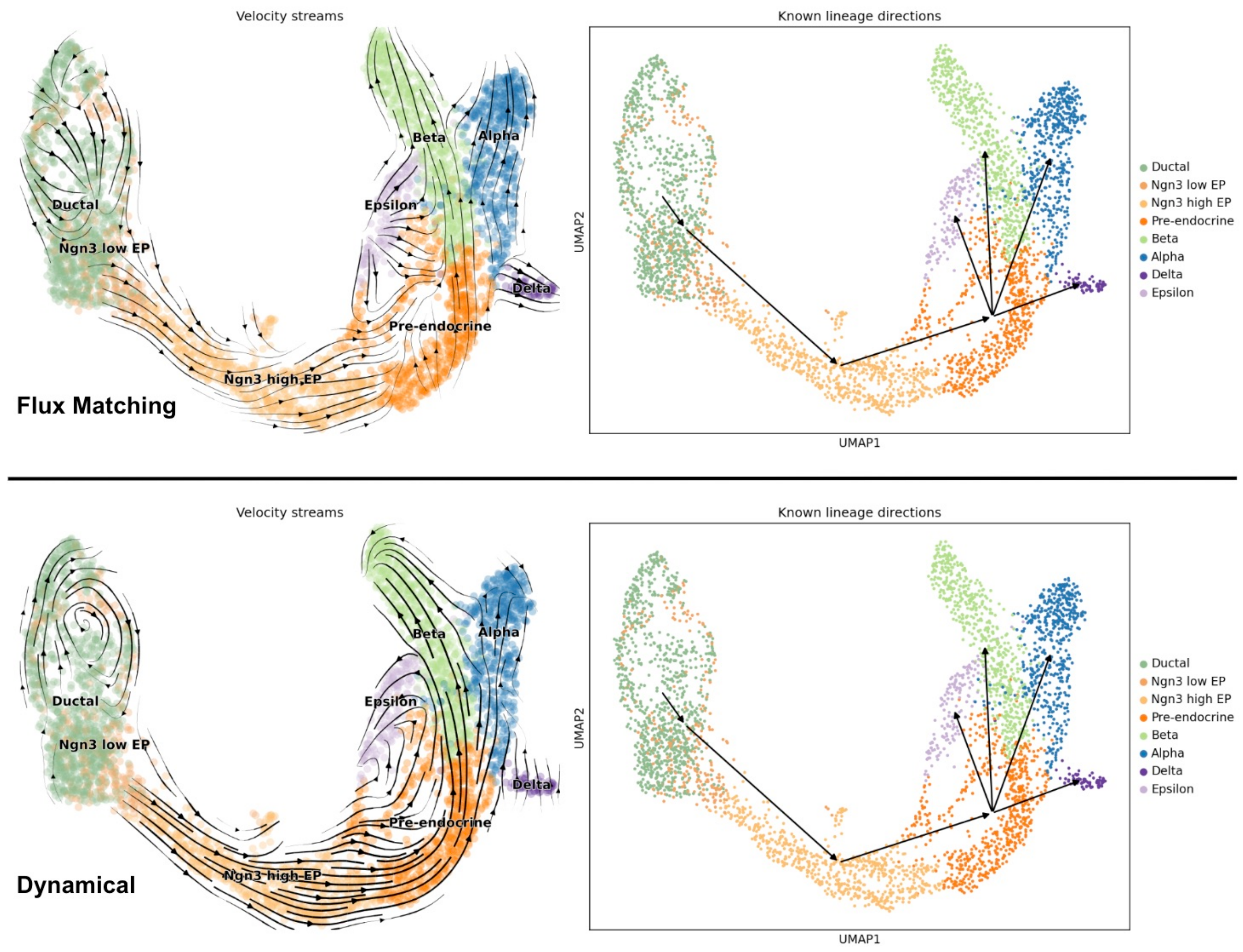}
  \caption{ \textbf{Pancreas Dataset}. (Left half) inferred RNA velocity (Right half) ground truth biological progression given by arrows}
  \label{rnavelo_pancreas_figure}
\end{figure}

\subsection{Application 3: \textit{Unrestricted} Generative Fields}
\label{app3_details_appendix}

\paragraph{Goal.}
While not the primary benefit of Flux Matching, we want to show in this experiment that Flux Matching can be used as a standalone generative model that performs well on high-dimensional, complex image distributions and scales efficiently in both runtime and memory.

\paragraph{Datasets.}
We evaluate on CIFAR-10 and CelebA. CIFAR-10 contains \(50{,}000\) training images with dimension \(3 \times 32 \times 32\). CelebA contains \(162{,}770\) training images with dimension \(3 \times 64 \times 64\).

\paragraph{Model parameterization.}
For both CIFAR-10 and CelebA, we use the standard U-Net architecture from \cite{futurexiang2023diffusion}. The architecture details are given in \Cref{tab:app3_model_hparams}.

\begin{table}[t]
\centering
\small
\setlength{\tabcolsep}{5pt}
\caption{Model parameterization for Experiment 3 \& 4.}
\label{tab:app3_model_hparams}
\begin{tabular}{@{}lll@{}}
\toprule
\textbf{Parameter} & \textbf{CIFAR-10} & \textbf{CelebA} \\
\midrule
Input shape & \(3 \times 32 \times 32\) & \(3 \times 64 \times 64\) \\
Base channels & \(128\) & \(128\) \\
Channel multipliers & \([1,2,2,2]\) & \([1,2,2,2]\) \\
Attention pattern & \([\texttt{False}, \texttt{True}, \texttt{False}, \texttt{False}]\) & \([\texttt{False}, \texttt{True}, \texttt{False}, \texttt{False}]\) \\
Dropout & \(0.1\) & \(0.1\) \\
Residual blocks per resolution & \(2\) & \(2\) \\
Number of parameters & \(35.7\)M & \(35.7\)M \\
\bottomrule
\end{tabular}
\end{table}

\paragraph{Training objective.}
We train using the EDM noise parameterization and distribution with default EDM settings on bfloat16. Since Flux Matching uses a batch KDE estimate of the score in its loss, we compare against the stable target version of diffusion models \cite{xu2023stable},which uses the same batch KDE score representation. This baseline is stronger than single sample DSM since stable targets were shown to outperform the standard single sample DSM by reducing variance. 

\paragraph{Training, sampling, and evaluation hyperparameters.}
We use the same hyperparameters for CIFAR-10 and CelebA except for the per-GPU batch size. During training, we compute \(10\)K-image FID every \(50{,}000\) iterations. For the final reported results, we select the checkpoint with the best \(10\)K-image FID and compute \(50\)K-image FID, negative log-likelihood, and Inception score (standard evaluation procedure, following \cite{song2020score}). The negative log-likelihood is computed using the same probability-flow ODE solver and tolerances used for sampling. The main training and sampling hyperparameters are given in \Cref{tab:app3_hparams}. Efficiency values reported in \Cref{tab:unrestricted_generation} were based on distributed training on $4$ NVIDIA A100 GPUs. Furthermore, to ensure a fair comparison of efficiency values for CIFAR10 and CelebA, we standardize the batch size to $32$ when computing wall-clock time and memory usage. 

\begin{table}[t]
\centering
\small
\setlength{\tabcolsep}{4pt}
\caption{Training and sampling hyperparameters for Experiments 3 \& 4. Unless otherwise noted, the same values are used for CIFAR-10 and CelebA. }
\label{tab:app3_hparams}
\begin{tabular}{@{}p{0.23\linewidth}p{0.22\linewidth}p{0.25\linewidth}p{0.20\linewidth}@{}}
\toprule
\multicolumn{2}{c}{\textbf{Training}} &
\multicolumn{2}{c}{\textbf{Sampling}} \\
\cmidrule(r){1-2}
\cmidrule(l){3-4}
\textbf{Hyperparameter} & \textbf{Value} &
\textbf{Hyperparameter} & \textbf{Value} \\
\midrule
Optimizer & AdamW &
Sampler & PF ODE \\

Learning rate & \(10^{-4}\) &
Solver & Adaptive second-order Heun \\

EMA rate & \(0.9993\) &
Implementation & \texttt{torchdiffeq.odeint} \\

Warmup steps & \(2{,}500\) &
Method & \texttt{adaptive\_heun} \\

Training iterations & \(500{,}000\) &
Relative tolerance & \(10^{-5}\) \\

Number of GPUs & \(4\) &
Absolute tolerance & \(10^{-5}\) \\

Batch size, CIFAR-10 & \(128\) per GPU &
\(\sigma_{\min}\) & \(0.002\) \\

Batch size, CelebA & \(32\) per GPU &
\(\sigma_{\max}\) & \(80\) \\

Horizontal flips & \(\texttt{True}\) &
Sampling interval & \([\sigma_{\min}, \sigma_{\max}]\) \\
\bottomrule
\end{tabular}
\end{table}

\paragraph{Extended samples.}
We show randomly generated CIFAR-10 samples in \Cref{cifar_samples} and CelebA samples in \Cref{celeba_samples}.

\begin{figure}
  \centering
  \includegraphics[width=\textwidth]{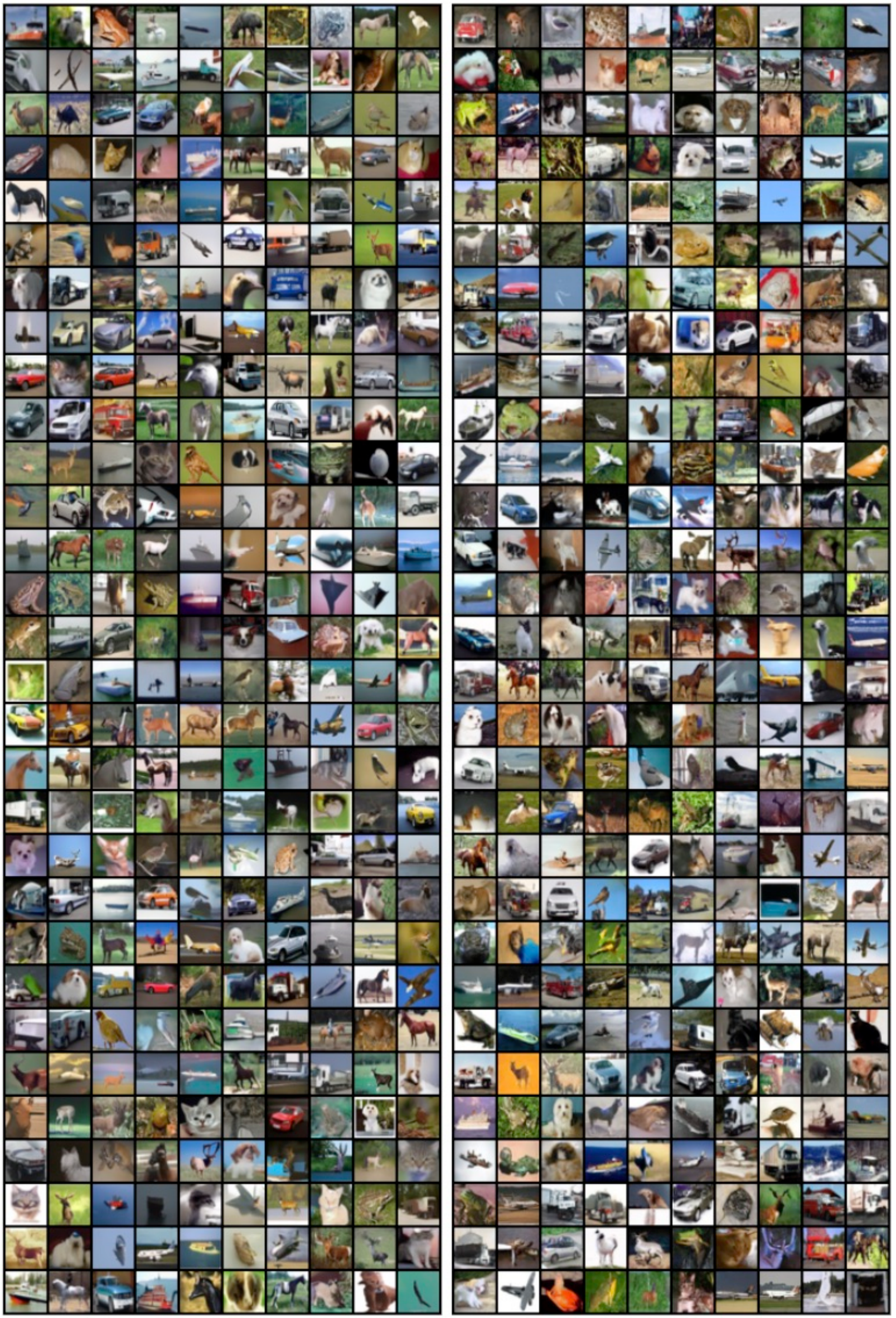}
  \caption{Random CIFAR-10 samples from Flux Matching (left) and DSM (right).}
  \label{cifar_samples}
\end{figure}

\begin{figure}
  \centering
  \includegraphics[width=\textwidth]{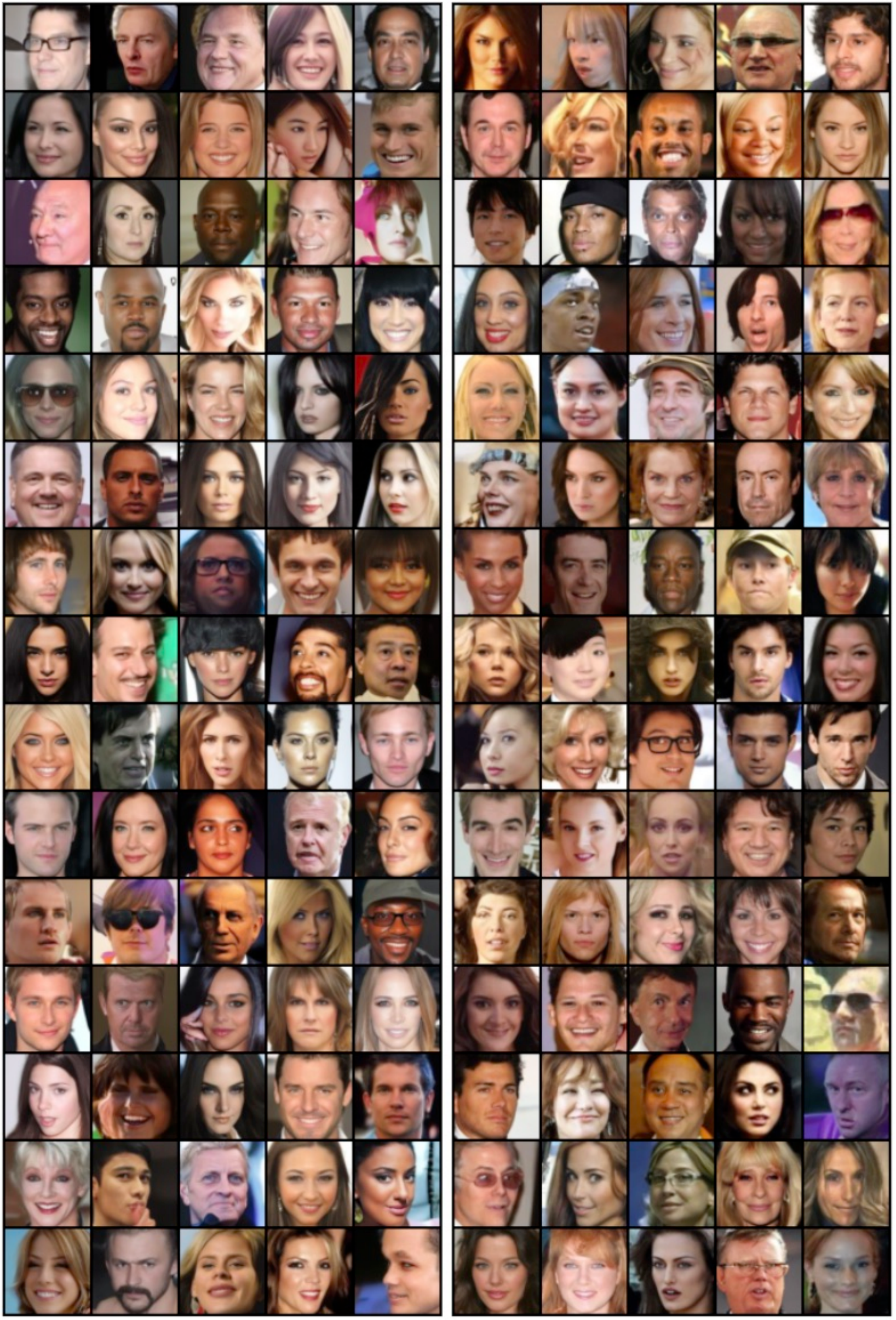}
  \caption{Random CelebA samples from Flux Matching (left) and DSM (right).}
  \label{celeba_samples}
\end{figure}

\subsection{Application 4: \textit{Fast Mixing} Generative Fields for Accelerated Sampling}
\label{app4_details_appendix}

\paragraph{Goal.} We want to show that Flux Matching is much more than just its unrestricted form (standalone objective), where optimizing different vector field attributes can have tangible benefits in purely generative modeling terms. Here, we show that we can reduce the number of sampling steps (aka number of function evalutions) by optimizing for vector fields with faster mixing properties.

\paragraph{Datasets, model parameterization, and evaluation.}
We use the same datasets, model parameterization, optimizer, training schedule, and noise distribution as in \Cref{app3_details_appendix}. For checkpoint selection, we follow the same protocol as in \Cref{app3_details_appendix}: for each method, we select the checkpoint with the best 10K FID using the adaptive Heun sampler. We then evaluate this selected checkpoint using the predictor--corrector sampler above and report 1K FID across different numbers of PC steps. These results are shown in \Cref{fast_mixing_figure}.

\paragraph{Training.}
Directly optimizing the mixing time of the sampler is intractable, so we introduce a tractable one-step proxy. Given a minibatch $\{x_0^{(i)}\}_{i=1}^{B}\sim p_\sigma$, we use the first half
$\{x_0^{(i)}\}_{i=1}^{B/2}$ as a reference batch from the target marginal.
From the second half, we construct an off-distribution batch by adding Gaussian noise $x_{\mathrm{noise}}^{(i)}
=
x_0^{(i+B/2)}+\sigma z_{0}^{(i)}$ where $z_{0}^{(i)}\sim\mathcal N(0,I)$ for $i=1,\dots,B/2$.
Starting from these off-distribution samples, we apply one Langevin step using
the learned field $f_\theta^\sigma$:
\begin{equation}
    x_{\mathrm{mix}}^{(i)}
=
x_{\mathrm{noise}}^{(i)}
+
h_\sigma f_\theta^\sigma(x_{\mathrm{noise}}^{(i)})
+
\sqrt{2h_\sigma}\,z_{\mathrm{noise}}^{(i)},
\qquad
z_{\mathrm{noise}}^{(i)}\sim\mathcal N(0,I),
\qquad
h_\sigma = 0.2\sigma^2.
\end{equation}
Our proxy for mixing speed measures whether this single step moves $x_{\mathrm{noise}}$ closer
to the target marginal:
\begin{equation}
\label{mixing_speed_proxy}
\mathcal{L}_{\mathrm{mixing}}
=
\mathbb{E}_{\sigma\sim\mathcal P}
\left[
\frac{
\mathrm{SD}
\!\left(
\{x_{\mathrm{mix}}^{(i)}\}_{i=1}^{B/2},
\{x_0^{(i)}\}_{i=1}^{B/2}
\right)
}{
\text{sg}\!\left[
\mathrm{SD}
\!\left(
\{x_{\mathrm{noise}}^{(i)}\}_{i=1}^{B/2},
\{x_0^{(i)}\}_{i=1}^{B/2}
\right)
\right]
}
\,e^{-s_{\mathrm{mixing}}(\sigma)}
+
s_{\mathrm{mixing}}(\sigma)
\right],
\end{equation}
where $\mathrm{SD}$ is the Sinkhorn divergence and $s_{\mathrm{mixing}}(\sigma)$ is a learned normalizer (similar as \Cref{noise_conditioned_section}, parameterized as a $1$-layer MLP and trained simultaneously with the main model) that keeps the mixing loss on a comparable scale across noise levels $\sigma$. We compute the Sinkhorn divergence with regularization parameter \(\varepsilon=0.05\) and \(50\) Sinkhorn iterations.

The final training objective is
\begin{equation}
\mathcal L_{\mathrm{flux\text{-}fast}}
=
\mathcal L_{\mathrm{flux\text{-}noise}}
+
\lambda_{\mathrm{mix}}\,\mathcal L_{\mathrm{mix}},
\end{equation}
with \(\lambda_{\mathrm{mix}}=0.01\).

\paragraph{Sampling.}
\begin{algorithm}[t]
\caption{Predictor--Corrector sampler \cite{song2020score}}
\label{alg:fast_mixing_pc_sampler}
\begin{algorithmic}[1]
\State Set \(\sigma_0=\sigma_{\max}>\cdots>\sigma_K=\sigma_{\min}\) on a logarithmic grid.
\State Sample \(x\sim\mathcal N(0,\sigma_{\max}^2 I)\).

\For{\(k=0,\ldots,K-1\)}
    \State \textbf{Predictor step}
    \State \(\Delta_k \gets \sigma_k^2-\sigma_{k+1}^2\).
    \If{\(k<K-1\)}
        \State Sample \(z\sim\mathcal N(0,I)\).
        \State \(x \gets x+\Delta_k f_\theta^{\sigma_k}(x)+\sqrt{\Delta_k}\,z\).
    \Else
        \State \(x \gets x+\Delta_k f_\theta^{\sigma_k}(x)\).
    \EndIf

    \If{\(k<K-1\)}
        \State \textbf{Corrector step}
        \State Sample \(z'\sim\mathcal N(0,I)\).
        \State \(\epsilon_k \gets 2\left(\rho\|z'\|_2/(\|f_\theta^{\sigma_{k+1}}(x)\|_2)\right)^2\), with \(\rho=0.16\) (Table 5 of \cite{song2020score}).
        \State \(x \gets x+\epsilon_k f_\theta^{\sigma_{k+1}}(x)+\sqrt{2\epsilon_k}\,z'\).
    \EndIf
\EndFor
\end{algorithmic}
\end{algorithm}
We evaluate fast mixing using the predictor--corrector (PC) sampler in
\Cref{alg:fast_mixing_pc_sampler}, following \cite{song2020score}. We use a
PC sampler rather than adaptive Heun because mixing speed concerns how quickly
the sampling dynamics induced by a field converge to their stationary
distribution. In the noised setting, the relevant stationary distribution at
noise level \(\sigma\) is the noised marginal \(p_\sigma\). Thus, a fast-mixing
field should be able to take a sample that is not exactly distributed according
to \(p_\sigma\) and quickly move it toward \(p_\sigma\). This is precisely
the role of the corrector step in PC sampling, while the predictor moves samples
across noise levels. When we use fewer sampling steps, these moves between noise levels become larger
and can move samples farther from the next noised marginal. The corrector then
uses the learned field at the current noise level to pull samples back toward
that marginal. Therefore, faster-mixing fields should enable fewer and larger predictor steps because even when the predictor introduces more error, the faster mixing corrector can remove that error more quickly.

For \(K\) PC steps, the sampler uses \(2K\) evaluations of
\(f_\theta^\sigma\): \(K\) predictor evaluations and \(K\) corrector
evaluations.

\subsection{Application 5: \textit{Embedding Structure} in Generative Fields}
\label{app5_details_appendix}
\begin{figure}
  \centering
  \includegraphics[width=0.8\textwidth]{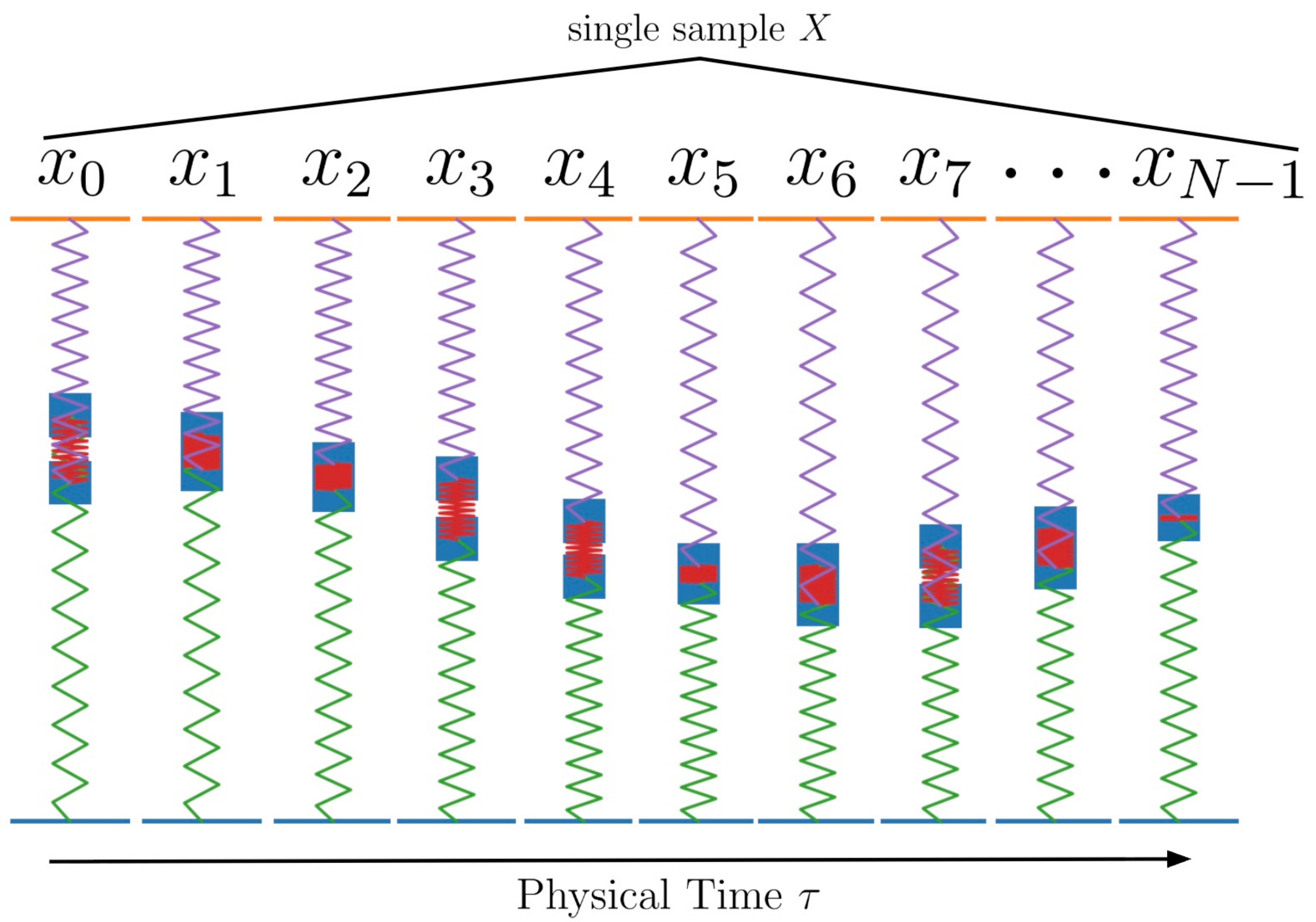}
  \caption{
  A single sample from the spring simulation. The system contains two masses, shown as blue blocks. Each mass is connected to a wall by a spring, shown in purple and green, and the two masses are connected to each other by a spring, shown in red. Each column corresponds to one physical time point and contains four features, which are the positions and velocities of the two masses. Collectively, these time points form one trajectory, which is one sample in the dataset.
  }
  \label{springs_illustration}
\end{figure}

\paragraph{Goal.} This experiment tests whether Flux Matching can incorporate directed structural relationships between variables while still generating all variables in parallel. The data is trajectories, so a natural inductive bias is temporal directionality---the field at a given physical time point should depend only on the current and previous time points. Standard diffusion samplers also generate all time points in parallel, but DSM restricts the learned field to be the score function, which have symmetric Jacobians by equality of mixed partial derivatives, making strictly directed dependencies such as temporal autoregression incompatible with the true score. Flux Matching has no such constraint, so we can impose a causal temporal mask on \(f_\theta^\sigma\), while still generating the entire trajectory in one parallel sampling procedure. This masking restricts the effective hypothesis class toward fields that respect the temporal ordering of the data, which can improve data efficiency \citep{baxter2000model}.

\paragraph{Dataset.}
\begin{table}[t]
\centering
\caption{Dataset and simulator parameters for the nonlinear spring experiment.}
\label{tab:springs_dataset_params}
\begin{tabular}{ll}
\toprule
Parameter & Value \\
\midrule
Number of training trajectories & \(2000\) \\
Trajectory length & \(N=50\) \\
State dimension per time point & \(4\) \\
Trajectory dimension after flattening & \(200\) \\
Wall spring constant & \(k_{\mathrm{wall}}=1.0\) \\
Coupling spring constant & \(k_{\mathrm{couple}}=0.7\) \\
Damping coefficient & \(\gamma=0.08\) \\
Nonlinear spring coefficient & \(c=0.08\) \\
Initial positions & \(q_i(0)\sim\mathcal N(0,3.0^2)\) \\
Initial velocities & \(v_i(0)\sim\mathcal N(0,1.5^2)\) \\
\bottomrule
\end{tabular}
\end{table}
Our dataset consists of simulated trajectories of two masses connected by nonlinear springs. Each data sample is a full trajectory over physical time. Let
\(x(\tau)=(q_1(\tau),v_1(\tau),q_2(\tau),v_2(\tau))\in\mathbb R^4\)
denote the simulator state at physical time \(\tau\), where \(q_i(\tau)\) and
\(v_i(\tau)\) are the position and velocity of mass \(i\). We record trajectories
\(X=(x_0,\ldots,x_{N-1})\in\mathbb R^{N\times 4}\), where
\(x_n:=x(n\Delta\tau)\). The trajectories are generated by integrating the ODE
below with RK4 using physical step size \(\Delta\tau=0.10\):
\begin{equation}
\label{spring_simulation_ode}
\begin{aligned}
\frac{d q_1}{d\tau} &= v_1, & \quad \frac{d v_1}{d\tau} &= -k_{\mathrm{wall}}q_1 -k_{\mathrm{couple}}(q_1-q_2) -\gamma v_1 -cq_1^3, \\
\frac{d q_2}{d\tau} &= v_2, & \quad \frac{d v_2}{d\tau} &= -k_{\mathrm{wall}}q_2 -k_{\mathrm{couple}}(q_2-q_1) -\gamma v_2 -cq_2^3 .
\end{aligned}
\end{equation}
We use the simulator and dataset parameters in \Cref{tab:springs_dataset_params}. \Cref{springs_illustration} visualizes one simulated trajectory.
\paragraph{Model parameterization.}
We use the same transformer architecture for Flux Matching and DSM. Each trajectory is represented as a sequence of physical-time tokens, where each token contains the state \((q_1,v_1,q_2,v_2)\). The model first projects each token to a hidden representation of $512$ dimensions, adds a sinusoidal positional embedding over physical time, and adds a learned embedding of the noise level \(\sigma\). The resulting sequence is processed by a stack of $8$ heads of $12$ residual self-attention blocks with LayerNorm, self-attention, and a GELU MLP. The final hidden states are normalized and projected back to the state dimension, producing one output vector per physical time point.

We compare two attention patterns. The noncausal model uses full self-attention across all physical time points. The causal model applies an upper-triangular attention mask over the physical-time dimension. In each attention block, if \(A_{ij}\) denotes the attention logit from time point \(i\) to time point \(j\), then the causal model sets
\begin{equation}
A_{ij}=-\infty
\qquad
\text{whenever}
\qquad
j>i
\end{equation}
before applying the softmax. As a result, the output at time point \(i\) can depend only on time points \(0,\ldots,i\) (visualized on the left side of \Cref{structured_fields_figure}). This enforces directed temporal dependence while preserving parallel generation, since all time points are still produced in a single forward pass.

\paragraph{Training objective.}
Since this experiment is not image-based, we use the standard variance-exploding (VE) noise distribution rather than the EDM noise distribution. Flux Matching is trained with the noised annealed Flux Matching loss in \Cref{noise_conditioned_section}. For noise annealed DSM, we use the stronger baseline of the stable-target variant from \cite{xu2023stable}, as was done in \Cref{app3_details_appendix}. 

We train four models: Flux Matching with full attention, Flux Matching with causal attention, DSM with full attention, and DSM with causal attention (all of which are noise annealed versions). This isolates the effect of adding a directed temporal mask under each training objective.

\paragraph{Training, sampling, and evaluation hyperparameters.}
All training and sampling hyperparameters are shown in \Cref{tab:springs_train_sample_eval_params}. At sampling time, we use the reverse VE sampler (Eq. 9 from \cite{song2020score}). We evaluate sample quality using Wasserstein distance \(\mathcal W_2\) between $2000$ generated trajectories and training trajectories.

\begin{table}[t]
\centering
\caption{Training and sampling hyperparameters for Experiment 5.}
\label{tab:springs_train_sample_eval_params}
\begin{tabular}{ll}
\toprule
Parameter & Value \\
\midrule
Noise sampling & \(\log\sigma\sim\mathrm{Unif}(\log\sigma_{\min},\log\sigma_{\max})\) \\
$\sigma_{\min}$ & \(0.002\) \\
$\sigma_{\max}$ & \(80.0\) \\
Optimizer & Adam \\
Training iterations & \(10,000\) \\
Batch size & \(256\) \\
Learning rate & \(3\times 10^{-4}\) \\
EMA rate & \(0.99\) \\
Sampler & Reverse VE sampler \citep{song2020score} \\
Sampling steps & \(512\) \\
\bottomrule
\end{tabular}
\end{table}

\section{Other Applications of Flux Matching}
We briefly outline additional potential applications of Flux Matching beyond those presented in the main text. These are directions we find exciting but were unable to pursue in this paper, and we hope they spark future research.
\subsection{Causality}
Vector fields, expressed as the drift of ordinary differential equations (ODEs) and stochastic differential equations (SDEs) \cite{rubenstein2016deterministic,hansen2014causal}, can be used to represent causal structure. Recent works \cite{lorch2024causal,bleile2026efficient} have explored learning generative vector fields to discover causal structures and perform inference by simulating the learned SDE. The core limitation of these works \cite{lorch2024causal,bleile2026efficient} is their inability to scale to high-dimensional data. For instance, \cite{lorch2024causal,bleile2026efficient} primarily evaluated on data with $20$ dimensions, far fewer than what is needed in settings like single-cell transcriptomics (via Perturb-seq \cite{dixit2016perturb}), which involves $\sim 20000$ dimensions. Notably, we can simply replace their loss function---KDS in \cite{lorch2024causal} and SKDS in \cite{bleile2026efficient}---with the Flux Matching loss to scale causal learning to very high dimensions, since, as shown in \Cref{unrestricted_generative_fields}, Flux Matching naturally scales to high-dimensional image datasets.

\subsection{Generative Modeling in Constrained Domains}
Prior work has studied diffusion models on constrained domains by either designing custom diffusion processes that respect the constraint set \(\Omega\) \cite{lou2023reflected} or by correcting invalid proposals through rejection or Metropolis-style steps when samples leave \(\Omega\) \cite{fishman2023metropolis}. Flux Matching could be a complementary approach. Even when the target distribution is supported on a constrained domain, there are many distribution generating vector fields that behave very differently near the boundary \(\partial \Omega\). Under a finite-step sampler, a field with large outward components near the boundary is more likely to produce invalid samples, while a field that is tangent to the boundary or points inward is more likely to remain inside the domain. Since Flux Matching does not require the learned field to equal the score, we can use the additional degrees of freedom to favor boundary-respecting dynamics. For example, the training objective could be designed so we penalize the learned field when it violates the support under the chosen discretization. 

\subsection{Spatially Structured Dependencies in Images}
Many generative problems have known structure among different parts of the sample like what we have outlined in \Cref{structured_generative_fields}. We focus here specifically on the application of images, where one example of structured dependencies mean allowing some regions of the image to influence others, while other interactions are not allowed. For example, the appearance of a human face may influence the appearance of sunglasses, but it should not directly alter the background. Flux Matching makes it possible to encode such region-to-region relationships directly in the architecture of the learned generative vector field (via, for example, masked attention). 

\subsection{Augmenting Existing Score-Based Models}
As shown in \Cref{flexibility_equation}, we can add any flux divergence-free term to $\nabla \log \pdata$ while still preserving the target distribution. Suppose we already have a trained off-the-shelf score-based model, and we want to leverage the benefits of Flux Matching (like accelerated sampling). Rather than training a new model from scratch using Flux Matching, we can train a network $v_{\phi}$ such that $\nabla \cdot (\pdata v_{\phi}) = 0$ using an analogous version of the Flux Matching loss:
\begin{equation}
\label{augmented_score_flux_loss}
    \mathcal L_{\mathrm{flux-aug}}(\theta)
:=
-\mathbb{E}_{\substack{t\sim q\\x_0\sim \pdata,\,x_t|x_0}}\!\left[
\frac{1}{q(t)}\,
v_\phi(x_0)^\top
\operatorname{sg}\!\left(
\frac{\partial x_t}{\partial x_0}^{\!\top} \nabla_{x_t} \mathcal{F}(x_t)
\right)
\right],
\end{equation}
where $\mathcal{F}(x_t) := \nabla \cdot v_\phi(x_t) + v_\phi(x_t) \cdot \nabla \log \pdata(x_t)$. Then, after $v_{\phi}$ is trained using $\mathcal L_{\mathrm{flux-aug}}(\theta)$, add $v_{\phi}$ to the already trained score model, $f_{\theta} = \nabla \log p + v_{\phi}$, and proceed to sample using $f_{\theta}$. 

This application can be useful in, for example, the case of faster mixing fields for accelerated sampling. We can simply train a fast mixing acceleration layer that we can "tack" onto an existing diffusion model we want to accelerate. 

\section{Flux Matching Details}
\label{noise_annealed_details_appendix}
\subsection{Component Calculation Specifications}
We provide an extended description of the different components in \Cref{loss_in_practice}. 
\subsubsection{Sampling MCMC horizon $t$ from importance sampler $q$}
\label{appendix_importance_sampler}
As mentioned in the main text, even though $q$ is supported on $[0,\infty)$, in practice, we find that defining the simulation horizon sampler $q$ to be either a truncated uniform or exponential on $[0,T]$ and setting $T = 4\sigma^2$ is sufficient. The bottom row of \Cref{mode_crossing_figure} supports this truncation: empirically, $\mathcal{L}_{\mathrm{flux}}$ decays approximately exponentially in $t$. Thus, after a sufficiently large cutoff, the remaining tail contribution is negligible. We find that setting $T = 4\sigma^2$ is sufficient, as shown in \Cref{mode_crossing_figure} since little mass remains after the red cutoff (set at $T = 4\sigma^2$). The simplest distribution for $q$ is the uniform density \(\mathcal{U}[0,T]\), but since the loss follows an exponential shape over $t$, a lower variance alternative is to sample from a truncated exponential distribution
\begin{equation}
q_\lambda(t)
=
\lambda e^{-\lambda t} \big/ \left(1-e^{-\lambda T}\right),
\qquad t\in[0,T],\quad \lambda>0,
\end{equation}
where $\lambda$ can be cheaply fitted during training as a single scalar parameter via

\begin{equation}
\label{q_single_aux_loss}
\mathcal L(\lambda)
:=
-
\mathbb E_{t\sim \bar q}
\left[
\operatorname{sg}\!\left(
\widehat{\mathcal L}_{\mathrm{flux}}(t)
\right)
\log q_{\lambda}(t)
\right].
\end{equation}
where \(\bar q\) is a fixed copy used to draw the current horizon $t$ and \(\widehat{\mathcal L}_{\mathrm{flux}}(t)\) is the realized Flux Matching loss. In the case of noise annealed Flux matching, we similarly learn
\(\lambda_\phi(\sigma)\), the rate of the truncated-exponential sampler $q_{\lambda_\phi(\sigma)}$ that is now dependent on $\sigma$. Let
\(\bar q_{\lambda_\phi(\sigma)}\) denote the fixed copy of this sampler used to draw
the current horizon \(t\), and let
\(\widehat{\mathcal L}_{\mathrm{flux}}^\sigma(t)\) be the realized noise annealed Flux
Matching loss, already importance reweighted by the sampling density
\(\bar q_{\lambda_\phi(\sigma)}\). We fit \(\lambda_\phi\) with 
\begin{equation}
\label{q_phi_aux_loss}
\mathcal L_q(\phi)
:=
-
\mathbb E_{\sigma\sim\mathcal P,\;t\sim \bar q_{\lambda_\phi(\sigma)}}
\left[
\operatorname{sg}\!\left(
\frac{\widehat{\mathcal L}_{\mathrm{flux}}^\sigma(t)}{\exp(s_\eta(\sigma))}
\right)
\log q_{\lambda_\phi(\sigma)}(t)
\right].
\end{equation}
where \(s_\eta(\sigma)\) is the learned normalizer (single-layer MLP) that reweighs losses from
different noise levels to be on comparable scales \cite{karras2024analyzing}. \(\lambda_\phi(\sigma)\) is also parameterized by a single-layer MLP and fitted simultaneously with the main network.

\subsubsection{Estimating $\nicefrac{\partial x_t}{\partial x_0}^{\top}\nabla_{x_t}r_\theta(x_t)$}
\label{appendix_variance_reduction_estimation}
We set weights $w_{ij}$ to be 
\begin{equation}
\label{weight_equation}
w_{ij}(t) = \operatorname{softmax}_{i}\!\left( -\nicefrac{\bigl\|x_t^{(j)} - x_0^{(i)} - \sigma^2(1-e^{-t})\nabla\log p_\sigma(x_0^{(i)})\bigr\|^2}{\bigl[2\sigma^2(1-e^{-2t})\bigr]} \right).
\end{equation}

\subsection{Variance at Intermediate Noise Levels}
\begin{figure}[t]
  \centering
  \includegraphics[width=\textwidth]{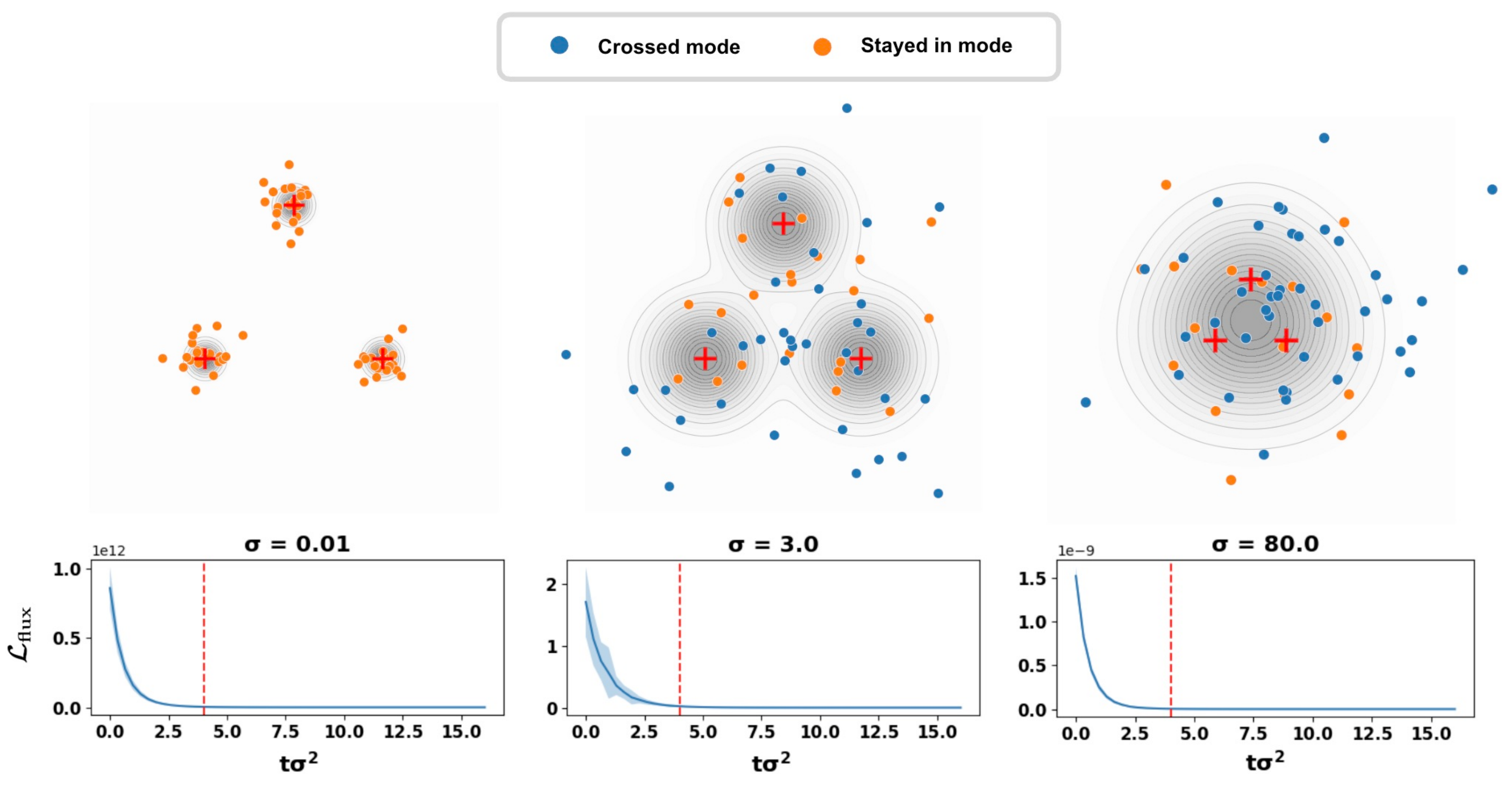}
  \caption{
  Mode crossing and loss variance in a three-component Gaussian mixture.
  The component means are fixed, while the component variance increases with
  \(\sigma\). (\textbf{Top row}) show terminal samples after running Langevin dynamics
  for \(100\) steps at different noise levels. Samples are colored blue if their
  terminal point is assigned to a different mode than their initial point, and
  orange otherwise. (\textbf{Bottom row}) shows the corresponding
  \(\mathcal{L}_{\mathrm{flux}}\) as a function of the simulation horizon, with
  shaded bands denoting standard deviation.
  }
  \label{mode_crossing_figure}
\end{figure}
In the noise annealed version of Flux Matching, we observe that the variance of
the objective can vary substantially across noise levels. The bottom row of
\Cref{mode_crossing_figure} illustrates this effect. At very low noise
\((\sigma=0.01)\) and very high noise \((\sigma=80.0)\), the empirical
\(\mathcal{L}_{\mathrm{flux}}\) has relatively small variance, as indicated by
the narrow standard-deviation bands. In contrast, at intermediate noise levels
\((\sigma \approx 3.0)\), the standard-deviation bands are much larger across
simulation horizons, suggesting that this regime produces a substantially
higher-variance estimator.

The top row of \Cref{mode_crossing_figure} provides intuition for this phenomenon. At low noise, the modes are well separated, and Langevin chains tend to remain in the same mode even after many simulation steps. At high noise, although chains mix more readily, the noised distribution is already close to a single Gaussian, so mode crossing is less informative and less variable. The intermediate regime is a difficult sweet spot where the effective supports of the modes begin to overlap, but the modes still remain distinct. As a result, some chains cross between modes while others do not, producing large sample-to-sample variation in the loss estimate. A related observation was also made by \cite{xu2023stable}, who found that intermediate noise levels can have higher variance in their learning objective.

This intermediate regime is also where the cross-chain minibatch estimator in
\Cref{flux_loss_learning_objective}, used to estimate
\(\nicefrac{\partial x_t}{\partial x_0}^{\top}
\nabla_{x_t} r_\theta(x_t)\), is most useful. By averaging information across
chains in the minibatch, the estimator is less sensitive to whether any single
chain crosses modes. In contrast, we found that the single-chain estimator
\(\nabla_{x_0} r_\theta(x_t)\) is often sufficient in the low- and high-noise
regimes, where the loss variance is much smaller.

\subsection{Architectures Need to Be Gradient Friendly}

The Flux Matching loss in \Cref{flux_matching_loss} requires differentiating
\(r_\theta\) with respect to the input. Since \(r_\theta\) itself contains a
divergence of the learned vector field, training depends on input derivatives of
\(f_\theta\) beyond the vector-field level. Thus, architectures used to
parameterize \(f_\theta\) should have well-behaved input gradients and
divergences. We did not study this issue systematically, but we found that using the specific architecture of ViT \cite{bao2023all} to parameterize $f_{\theta}$ does not work well (while UNet did) in Flux Matching. We hypothesize that the patching procedure used in ViT may be problematic, but this remains an open question for future work.

\section{\texorpdfstring{$v$}{v}-Flux Matching for Distribution Flows}
\label{v_flux_matching_appendix}

\begin{figure}
  \centering
  \includegraphics[width=\textwidth]{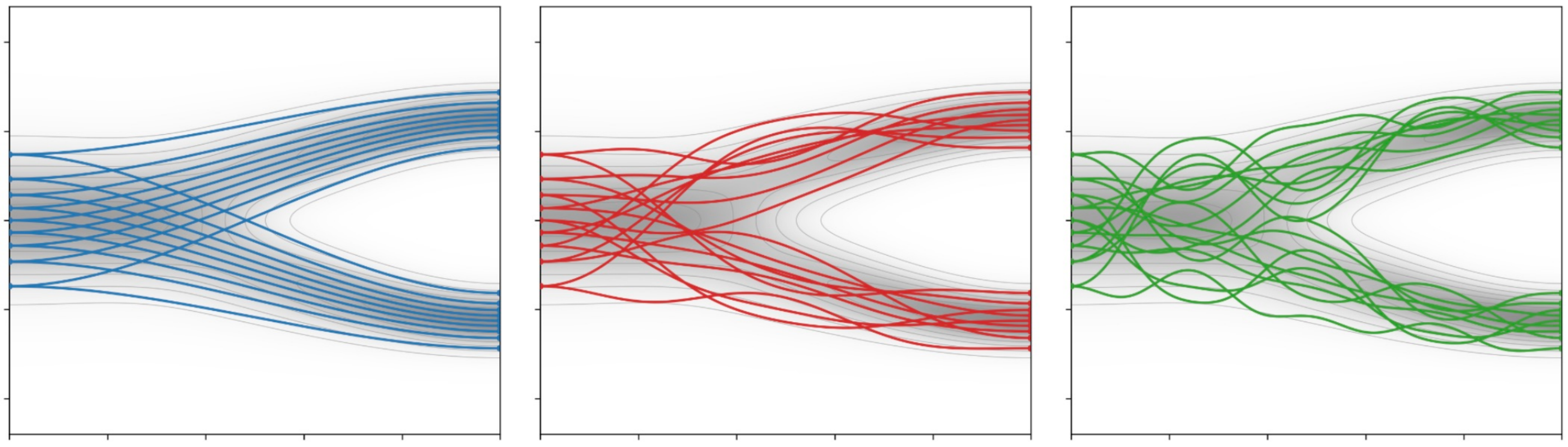}
  \caption{
  Distribution flows can have the same marginal evolution but different particle
  trajectories. The gray background shows the same path of marginals, while the
  solid colored curves show different individual transports that realize this
  path.
  }
  \label{distribution_flow_figure}
\end{figure}

Flux Matching is not restricted to matching the score. More generally, it can
match the flux divergence induced by any vector field. This allows us to start
from a distribution flow model, such as flow matching
\cite{lipman2022flow,tong2023improving} or related constructions
\cite{neklyudov2023action,albergo2022building,liu2022flow}, and learn
alternative particle dynamics that preserve the same marginal path.

\begin{blueemphbox}
\textbf{New Capability: \textit{Same} Marginal Path, \textit{Many} Particle
Paths}. Flux Matching can learn different individual transports while preserving
the same evolution of distributions.
\end{blueemphbox}

\subsection{Distribution Flows and Equivalent Vector Fields}
\label{distribution_flow_equivalence}

Let \(a\in[0,1]\) denote distribution flow time, and let
\(\{p_a\}_{a\in[0,1]}\) be a path of densities. A velocity field \(v_a\)
induces this path if it satisfies the continuity equation
\begin{equation}
\label{continuity_equation}
\frac{\partial p_a(x)}{\partial a}
=
-\nabla\cdot\left(p_a(x)v_a(x)\right).
\end{equation}
The velocity field that realizes a given marginal path is not unique: any two vector fields \(v_a\) and \(f_a\) induce the same instantaneous
distributional change whenever
\begin{equation}
\label{same_flux_divergence_condition}
\nabla\cdot\left(p_a f_a\right)
=
\nabla\cdot\left(p_a v_a\right).
\end{equation}
Assume $\nabla\cdot\left(p_a f_a\right) = \nabla\cdot\left(p_a v_a\right)$ for all $a \in [0,1]$. If \(v_a\) and \(f_a\) are initialized with the same density \(p_0\), they induce the same marginal path \(\{p_a\}_{a\in[0,1]}\), and in particular the
same terminal distribution \(p_1\).

\subsection{\texorpdfstring{$v$}{v}-Flux Matching at a Single Marginal}
\label{single_v_flux_matching}

For clarity, we first fix a distribution flow time \(a\) and a reference velocity field \(v_a\). Let \((x_t)_{t\ge0}\) denote the diffusion $dx_t=\nabla\log p_a(x_t)\,dt + \sqrt{2}\,dW_t$ with stationary density \(p_a\).  To emphasize, the diffusion is only used to obtain the correct geometry on the loss; the target vector field being flux matched is \(v_a\), not \(\nabla\log p_a\). Let
\begin{equation}
\label{v_flux_residual}
u_{\theta,a}(x)
=
f_\theta(x,a)-v_a(x),
\qquad
r_{\theta,a}(x)
=
\nabla\cdot u_{\theta,a}(x)
+
u_{\theta,a}(x)\cdot\nabla\log p_a(x).
\end{equation}
Define
\begin{equation}
\label{v_flux_matching_loss}
\mathcal L_{\mathrm{v\text{-}flux}}^a(\theta)
:=
-
\mathbb{E}_{\substack{t\sim q\\x_0\sim p_a,\;x_t|x_0}}
\left[
\frac{1}{q(t)}
u_{\theta,a}(x_0)^\top
\operatorname{sg}\left(
\frac{\partial x_t}{\partial x_0}^{\top}
\nabla_{x_t}r_{\theta,a}(x_t)
\right)
\right].
\end{equation}
This is the same objective as \Cref{flux_matching_loss}, with \(\pdata\) replaced by
\(p_a\) and the score target replaced by \(v_a\).

\begin{graypropbox}
\begin{restatable}[\(v\)-Flux Matching]{corollary}{vfluxcorollary}
\label{cor:v_flux_matching}
Fix \(a\in[0,1]\). Assume \(p_a>0\) on \(\mathbb R^d\) and boundary terms in
integration-by-parts arguments vanish. Let
\(\Pi_{\mathrm{flux}}^{p_a}\) follow the same definition as \Cref{grad_proj_fisher_div} with respect to \(p_a\). Let
\begin{equation}
\label{v_projected_fisher_divergence}
\widetilde{\mathcal J}_{\mathrm{v\text{-}flux}}^a(\theta)
:=
\mathbb E_{x\sim p_a}
\left[
\left\|
\Pi_{\mathrm{flux}}^{p_a} f_\theta(\cdot,a)(x)
-
\Pi_{\mathrm{flux}}^{p_a} v_a(x)
\right\|^2
\right].
\end{equation}
Then
\begin{equation}
\label{v_flux_matching_gradient_equivalence}
\nabla_\theta
\widetilde{\mathcal J}_{\mathrm{v\text{-}flux}}^a(\theta)
=
2\nabla_\theta
\mathcal L_{\mathrm{v\text{-}flux}}^a(\theta).
\end{equation}
\end{restatable}
\end{graypropbox}

\begin{proof}
This is a direct application of \Cref{theorem1}. The proof of
\Cref{theorem1} only depends on the density \(\pdata\) and the mismatch field
\(u_\theta\). Setting \(\pdata=p_a\) and
\(u_\theta=u_{\theta,a}=f_\theta(\cdot,a)-v_a\) gives exactly
\Cref{v_flux_matching_loss} and
\Cref{v_projected_fisher_divergence}.
\end{proof}

\textbf{Marginal Velocity for Flow Matching Paths}. The marginal velocity depends on the path used to define the interpolating
marginals \(p_a\). As a simple example, consider the conditional flow matching
path of \cite{lipman2022flow} with \(x_1\sim p_1\) and
\begin{equation}
p_a(x\mid x_1)
=
\mathcal N\!\left(x;\,a x_1,\,(1-a)^2 I\right).
\end{equation}
Equivalently, \(x_a=a x_1+(1-a)\epsilon\) with \(\epsilon\sim\mathcal N(0,I)\).
The marginal velocity can be approximated by the minibatch
\(\{x_1^{(i)}\}_{i=1}^B\), which gives
\begin{equation}
v_a(x)
\approx
\sum_{i=1}^B
w_i(x,a)\,
\frac{x_1^{(i)}-x}{1-a},
\qquad
w_i(x,a)
=
\frac{p_a(x\mid x_1^{(i)})}
{\sum_{j=1}^B p_a(x\mid x_1^{(j)})}.
\end{equation}
which is a common approximation used in few-step generative models (e.g. \cite{geng2025mean}).

\subsection{Pathwise \texorpdfstring{$v$}{v}-Flux Matching}
\label{pathwise_v_flux_matching}

The previous subsection defined \(v\)-Flux Matching at a fixed marginal \(p_a\).
To match an entire distribution flow, we apply the same objective independently
along the path \(\{p_a\}_{a\in[0,1]}\). The pathwise \(v\)-Flux Matching
objective is
\begin{equation}
\label{pathwise_v_flux_matching_loss}
\mathcal L_{\mathrm{path\text{-}v\text{-}flux}}(\theta,\phi,\eta)
:=
\mathbb E_{a\sim \nu}
\left[
\mathcal L_{\mathrm{v\text{-}flux}}^a(\theta,\phi)
/ \exp(s_\eta(a))
+
s_\eta(a)
\right],
\end{equation}
where \(\nu\) is the sampling distribution over flow times \(a\in[0,1]\), and
\(\mathcal L_{\mathrm{v\text{-}flux}}^a(\theta,\phi)\) denotes the fixed-\(a\)
loss in \Cref{v_flux_matching_loss} with \(f_\theta^a(x):=f_\theta(x,a)\) and
diffusion simulation horizon distribution \(q_{\lambda_\phi(a)}\). All diffusion simulations are
performed with the score \(\nabla\log p_a\). The learned normalizer
\(s_\eta(a)\) reweighs losses from different marginals so that they remain on
comparable scales \cite{karras2024analyzing}, while \(\lambda_\phi(a)\)
parameterizes the diffusion horizon sampler at each point along the path.

\section{Related Work}

\subsection{\textit{Analyzing} Non-Conservative Generative Vector Fields}
A long line of work from the MCMC and statistical physics literature has analyzed the space of vector fields that preserve a given stationary distribution and characterized the benefits of departing from reversibility. Augmenting the score with a divergence-free flux component has been shown to accelerate convergence to the stationary distribution \cite{hwang2005accelerating, rey2015irreversible} (which we empirically corroborate in \Cref{fast_mixing_generative_fields}) and to reduce the asymptotic variance of resulting estimators \cite{duncan2016variance, duncan2017nonreversible}. \cite{ma2015complete} gives a complete parameterization of all continuous Markov processes admitting a prescribed stationary distribution. Crucially, these works assume a closed-form target distribution and prescribe the non-conservative drift analytically rather than learning it from data. Flux Matching is, to our knowledge, the first objective to operationalize this body of theory into a learning objective for non-conservative generative dynamics that scales to high-dimensional distributions.

\subsection{\textit{Learning} Non-Conservative Generative Vector Fields}
A handful of prior works have explored learning non-score generative vector fields. \cite{zhang2025equilibrium} learns a non-conservative vector field to assign dynamics to snapshot data, similar to our RNA velocity experiment, but their method is restricted to $2$D. The objectives of \cite{lorch2024causal,bleile2026efficient}, while originally proposed for causal SDEs, could in principle learn arbitrary generative vector fields. However, neither scales beyond $20$D. Flux Matching is novel not by the goal of learning non-gradient fields, but by being the first objective to do so while scaling to high-dimensional, complex distributions.

\cite{petrovic2025curly} frame their method as learning non-gradient field dynamics, yet it fundamentally matches a prior to a terminal distribution via a learned (rather than predefined) interpolation, a special case of \cite{neklyudov2023computational}. Crucially, their method cannot learn non-gradient fields given a single distribution. The two approaches are in fact complementary, with \cite{petrovic2025curly} producing a bridge of distributions and Flux Matching providing flexibility over the individual particle trajectories that realize the given bridge.

\subsection{Enforcing the Fokker--Planck (or Continuity) Equation}
\cite{lai2023fp} and \cite{huang2026improving} enforce the Fokker--Planck equation (respectively, the continuity equation) as a regularizer on top of a primary score matching or flow matching objective, providing tighter control over the induced PDE. In contrast, Flux Matching is a standalone generative objective: \cite{lai2023fp,huang2026improving} add a regularizer to a generative loss, whereas Flux Matching \emph{is} the generative loss. Moreover, both prior methods are restricted to score matching/flow matching-style trajectories, while Flux Matching admits arbitrary trajectories realizing the same marginals.

Via the Fokker--Planck, \cite{horvat2024gauge} observe that diffusion models possess a gauge degree of freedom (as we formalize in \Cref{prop1}), but treat this purely as an observation with no associated learning procedure. 

\section{Limitations}
\label{limitations}
Flux Matching is most useful when the goal is not only to model a distribution, but also to learn a generative vector field with additional desired structure. If one only cares about unrestricted generative modeling, then standard DSM already provides a highly optimized objective. Flux Matching should be viewed less as a replacement for DSM and more as a paradigm that enables use cases where the vector field itself matters. The main practical limitation is computational cost, with our implementation being roughly $2\text{--}4\times$ more expensive than DSM in both runtime and memory. Reducing this overhead is an important direction for future work.

Further, the flexibility that Flux Matching provides is only useful when paired with a meaningful inductive bias or auxiliary objective. Flux Matching expands the class of learnable generative vector fields but does not automatically identify the best member for a given application, and designing architectures or regularizers that exploit this extra freedom remains problem dependent. Finding clever and elegant ways of constraining the learned vector field to exhibit a desired attribute is itself the central task facing a practitioner using Flux Matching. We expect that individual instantiations of such constructions---each tailored to a particular structural or application-specific goal---can constitute a substantial contribution in their own right.


\end{document}